\documentclass{ecai} 
\usepackage{latexsym}
\usepackage{amssymb}
\usepackage{amsmath}
\usepackage{amsthm}
\usepackage{booktabs}
\usepackage{enumitem}
\usepackage{graphicx}
\usepackage{color}
\usepackage{bbm}

\usepackage{algorithm}
\usepackage{algorithmic}

\usepackage{subfigure}
\usepackage{dsbda-style}
\newcommand{\norm}[1]{\left\lVert#1\right\rVert}
\usepackage{adjustbox}
\usepackage{makecell}
\usepackage{caption} 

\newtheorem{theorem}{Theorem}
\newtheorem{lemma}[theorem]{Lemma}

\newcommand{\BibTeX}{B\kern-.05em{\sc i\kern-.025em b}\kern-.08em\TeX}

\begin{document}

%%%%%%%%%%%%%%%%%%%%%%%%%%%%%%%%%%%%%%%%%%%%%%%%%%%%%%%%%%%%%%%%%%%%%%%%

\begin{frontmatter}

%%% Use this command to specify your submission number.
%%% In doubleblind mode, it will be printed on the first page.

\paperid{4315} 

%%% Use this command to specify the title of your paper.

\title{Gumbel-MPNN: Graph Rewiring with Gumbel-Softmax}

%%% Use this combinations of commands to specify all authors of your 
%%% paper. Use \fnms{} and \snm{} to indicate everyone's first names 
%%% and surname. This will help the publisher with indexing the 
%%% proceedings. Please use a reasonable approximation in case your 
%%% name does not neatly split into "first names" and "surname".
%%% Specifying your ORCID digital identifier is optional. 
%%% Use the \thanks{} command to indicate one or more corresponding 
%%% authors and their email address(es). If so desired, you can specify
%%% author contributions using the \footnote{} command.

\author[A]{\fnms{Marcel}~\snm{Hoffmann}\orcid{0000-0001-8061-9396}\thanks{Corresponding Author. Email: marcel.hoffmann@uni-ulm.de}}
\author[B]{\fnms{Lukas}~\snm{Galke}\orcid{0000-0001-6124-1092}}
\author[A]{\fnms{Ansgar}~\snm{Scherp}\orcid{0000-0002-2653-9245}} 

\address[A]{Ulm University}
\address[B]{University of Southern Denmark}

%%% Use this environment to include an abstract of your paper.

\begin{abstract}
Graph homophily has been considered an essential property for message-passing neural networks (MPNN) in node classification.
Recent findings suggest that performance is more closely tied to the consistency of neighborhood class distributions.
We demonstrate that the MPNN performance depends on the number of components of the overall neighborhood distribution within a class.
By breaking down the classes into their neighborhood distribution components, we increase measures of neighborhood distribution informativeness but do not observe an improvement in MPNN performance.
We propose a Gumbel-Softmax-based rewiring method that reduces deviations in neighborhood distributions.
Our results show that our new method enhances neighborhood informativeness, handles long-range dependencies, mitigates oversquashing, and increases the classification performance of the MPNN.
The code is available at \url{https://github.com/Bobowner/Gumbel-Softmax-MPNN}.

\end{abstract}

\end{frontmatter}

%%%%%%%%%%%%%%%%%%%%%%%%%%%%%%%%%%%%%%%%%%%%%%%%%%%%%%%%%%%%%%%%%%%%%%%%

\section{Introduction}

Graph Neural Networks (GNNs)~\cite{org_gnn}, particularly Message Passing Neural Networks~\cite{GCN, GAT, GraphSage} (MPNNs), are state-of-the-art in node classification tasks~\cite{strong_baselines, HomophilyNecessity}.
Established MPNNs are designed under the homophily assumption, \ie the idea that connected nodes share similar properties and belong to the same class.
For example, in citation graphs, papers frequently cite within the same field, or in social networks, similar people tend to form connections~\cite{H2GNN}.

MPNNs have been believed to perform well on homophilic graphs but poorly on heterophilic graphs, such as financial transaction networks, where fraudsters are more likely to connect with regular users than with one another fraudster~\cite{CPGNN, BlockModellingGCN, HOG-GNN}.
Recent studies have challenged this assumption, demonstrating that MPNNs can achieve strong performance not only on highly homophilic graphs but also on highly heterophilic graphs and graphs where the neighborhood distribution is consistent within the class~\cite{HomophilyNecessity, NeighborhoodCrucial}. 
\citet{PerfLocalHomophily} showed that node misclassification often arises due to a discrepancy between local homophily, \ie how similar a node is to its neighbors, and the global homophily of its class, \ie the average homophily of all nodes belonging to that class. 
Nevertheless, most models specifically designed for heterophilic graphs fail to account for both the consistency of neighborhood distributions and the divergence between local and global homophily~\cite{H2GNN, GeomGNN, GPR-GNN}.
%This gap highlights a fundamental challenge in designing MPNNs that can effectively generalize across varying levels of homophily and heterophily, necessitating a better understanding of the neighborhood structures and class distributions for graph learning. %Does not really fit we propose more of a method to train MPNNs than a MPNN

We tackle this challenge by adjusting the edges to obtain more consistent neighborhood distributions. 
Current graph rewiring techniques, \ie techniques to adapt the adjacency matrix to facilitate the downstream task, are typically tailored for tasks using collections of small graphs, \eg graph classification, and fail to scale to a large number of nodes commonly encountered in node classification~\cite{LASER, ProbRewiring}.
Others disregard neighborhood distribution consistency altogether ~\cite{Ricci_Rewiring, DHGR}.

We show that the neighborhood structure of a class can be represented as a combination of $k$ distinct neighborhood distributions; see Figure~\ref{fig:example} for illustration.
Our theoretical and empirical analyses reveal that neighborhood distributions harm the performance of MPNNs in two cases: (A) when the distributions are too dissimilar from each other and (B) when the neighborhood-mixture distribution decomposes into too many components.
We enhance the consistency of the mixed distributions by assigning new, separate class labels to each distribution component.
We observe that the increase in consistency is reflected by the Label Informativeness (LI)~\cite{LabelInformativeness} metric, which is the mutual information between a node's label and the label distribution of its neighbors.
However, we see that an increase in LI does not necessarily improve the downstream task performance of the MPNN.
The reason is that LI does not consider the aggregated feature vectors. % as illustrated in Figure~\ref{fig:example}.
A method is required to increase the consistency of neighborhood distributions within each class.

\begin{figure}[!htbp]
    \centering
    \subfigure[Graph with node labels.]{
        \includegraphics[width=0.38\linewidth]{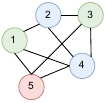}
        \label{fig:subfig1}
    }
    \hspace{0.5cm}
    \subfigure[
    Mixture embeddings after aggregation.]{
        \includegraphics[width=0.38\linewidth]{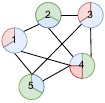}
        \label{fig:subfig2}
    }
    \caption{Problem illustration to show the idea of neighborhood-based aggregation. On the left, the graph has labels in colors. 
    After aggregation (without self-loops), each node embedding is a mixture of its neighbors' embeddings, which has been distributed according to the feature distribution of their class. 
    The green class (nodes 1 and 3) is easy to classify since their embedding after aggregation is similar. 
    The blue class (nodes 2 and 4) is difficult to classify since the neighborhood is a mixture of two distributions, which leads to dissimilar embeddings after aggregation.}
    \label{fig:example}
\end{figure}

We propose Gumbel-MPNN, an end-to-end differentiable rewiring model based on Gumbel-Softmax, which modifies the edges of the graph to obtain more consistent neighborhood distributions per class.
Our model can effectively increase the consistency of the neighborhood distributions for all classes, leading to better downstream task performance.
It reduces edge noise by removing noisy edges, handles long-range dependencies, and reduces oversquashing by adding critical connections through bottlenecks.
By pre-selecting promising candidate edges, we avoid considering all pairs of nodes, allowing our rewiring model to scale to large graphs.
% new, saves a line
% In summary, our contributions are:  
Our key contributions are as follows:  
\begin{itemize}  
    \item %Analysis of Neighborhood Distributions
    We provide a theoretical and empirical analysis of the neighborhood distributions and their impact on the performance of MPNNs.  
    \item %Novel Rewiring Model
    We introduce a scalable, trainable rewiring approach 
    Gumbel-MPNN
    that reduces the standard deviation of the neighborhood distribution per class by up to $10\%$.  
    \item %Enhanced Robustness
    Our model reduces edge noise, improves robustness to neighborhood variations, and captures dependencies beyond the number of message-passing layers.

    \item %Can be removed if space is needed
    We conduct extensive experiments on $12$ benchmark datasets using $6$ baseline models  with a fair hyperparameter optimization where our model is on-par with all baselines. 
    \end{itemize}

%In summary, our contributions are:

%\begin{itemize}
%    \item A theoretical and empirical analysis of the influence of neighborhood distributions on GNN performance
%    \item A new, trainable and scalable rewiring model that decreases the standard deviation of the neighborhood distribution per class by up to $10\%$
    %\item We show (hopefully) that our model harmonizes neighborhood structures for specific classes
%    \item An empirical analysis that confirms our model reduces edge noise at test time, is more robust towards multiple components of neighborhood distributions in a class, and can model dependencies which are longer than the number of MPNN layers
%    \item Extensive experiments on $12$ benchmark datasets with $6$ baselines to show that our model is on par with existing baselines.
%\end{itemize}

%Below, we summarize the related work.
%Section~\ref{sec:preleminaries} provides the preliminaries.
%Section~\ref{sec:NeighDistImpact} describes the challenges of inconsistent neighborhood distributions.
%Section~\ref{sec:methods} introduces our rewriting model.
%The experiments are presented in Section~\ref{sec:experiments} and discussed in Section ~\ref{sec:discussion}, before we conclude.

\section{Related Work}
\label{sec:relatedwork}
We discuss related work on graph neural networks, the analysis of heterophily in graphs and existing rewiring approaches on graph representation learning.

\textbf{Graph Neural Networks and Heterophily.}
Graph Neural Networks (GNNs), especially Message Passing Neural Networks (MPNNs)~\cite{GCN, GAT, GraphSage}, represent the state-of-the-art in most graph-related tasks.
However, many works assume that standard MPNNs struggle on heterophilic graphs~\cite{CPGNN, BlockModellingGCN, HOG-GNN}.
Recent work has shown that this is not necessarily true~\cite{HomophilyNecessity, NeighborhoodCrucial}.
Standard MPNNs can perform well on such datasets if the node neighborhood distribution for a specific class is consistent within this class~\cite{HomophilyNecessity}.

%\textbf{Understanding Heterophily.}
The impact of heterophily in MPNNs has been studied from various perspectives.
\citet{EdgeDir} showed that the direction of the edges is crucial in heterophilic graphs and modified MPNNs such that they explicitly differentiate between incoming and outgoing edges.
\citet{PerfLocalHomophily} showed that the nodes where the local homophily differs from the global homophily of the graph are difficult to classify.
\citet{StatisticalTextHomophily} analyzed the effect of inter- and intra-class feature separability for multiple homophily levels and showed that their ratio is a good predictor for MPNN performance.
Other studies show a connection of the heterophily problem to the oversmoothing problem~\cite{Heterophily+Oversmoothing, curvature_rewireing} and that shuffling feature vectors between nodes of the same class increases the generalization of MPNNs~\cite{FeatureShuffle}.

\textbf{Specialized MPNNs for Heterophilic Graphs.}
Two main lines of work have developed specialized models for heterophilic graphs.
We first discuss modifications of the aggregation mechanism, and second, methods that model the heterophily explicitly in the learning process.

%Aggregation modification
First, \citet{GeomGNN} proposed GeomGCN with a second neighborhood in a latent space and applied a bi-level aggregation scheme to combine it with the graph neighborhood.
\citet{H2GNN} combined ego and neighborhood separation with higher-order neighborhood aggregation.% in $H_2GCN$.
\citet{FSGNN} showed that some features are detrimental during the aggregation and proposed a soft feature selection method to alleviate this issue.
\citet{GBK-GNN} learned a bi-kernel aggregation that amplifies or inhibits neighbor nodes based on their level of homophily.
He~et~al.~\cite{BlockModellingGCN} learned different aggregations by block matrices for different nodes.
\citet{JacobiConv} proposed JacobiConv, which uses an orthogonal basis of the graph Laplacian to adapt to heterophilic and homophilic graphs.
\citet{HOG-GNN} learned adapting MPNN aggregation based on a learned homophily matrix.
\citet{FAGCN} increased the frequency spectrum that the MPNN can learn to aggregate, which improves the integration of information from heterophilic neighbors.
\citet{ACM-GNN} proposed an adaptive channel mixing framework to learn how to mix low and high-frequency signals or just use the identity.
\citet{PolyGCL} learned a linear combination of low and high-pass graph filters to handle homophilic and heterophilic graphs.
\citet{NeuralSheaf} learned cellular sheaves, which generalized the learnable graph geometry.

%Heterophily in the learning process
Second, \citet{GPR-GNN} proposed a PageRank-based MPNN architecture to learn weights that adapt to multiple node label structures in the aggregation.
\citet{CPGNN} models the homophily and heterophily as a learned compatibility matrix to increase the performance on graphs with arbitrary homophily levels.
\citet{GloGNN} computed node embeddings by aggregating messages from global nodes to increase the ratio of homophilic nodes in the neighborhood.
\citet{DHGR} preprocess the graph to increase the homophily by adding and deleting edges based on the node feature similarity.
\citet{NeighborhoodCrucial} proposed a neighborhood class consistency metric that is transformed into a learning objective to modify the adjacency matrix with a graph auto-encoder.
\citet{DisamGCL} identified heterophilic nodes by inconsistent predictions during training and used them for contrastive learning with nodes in their neighborhood to increase their discriminability.
%\todopink{we see a long enumeration of things; but not reflections.}

While these methods can enhance performance on heterophilic graphs, they fail to account for neighborhood consistency and do not interpret the neighborhood as a mixture of distributions.

\textbf{Graph Rewiring.}
Various problems in graph representation learning have been approached by graph rewiring, \ie adapting the adjacency matrix to facilitate the task, \eg heterophilic graph learning~\cite{DHGR}, oversmoothing~\cite{DropEdge}, oversquashing~\cite{LASER}, or graph expressivity enhancement~\cite{ProbRewiring}. 
\citet{AssortivityMixingCompGraph} compute a new graph based on pairwise structural node similarity to increase the MPNN performance.
\citet{ProbRewiring} learned to sample a new adjacency matrix to alleviate oversquashing and underreaching.
\citet{LatentGraphInferenceCW} can learn a new input graph based on cellular complexes. 
\citet{LASER} proposed a sequential rewiring technique to overcome oversquashing and underreaching in graph-level tasks.
\citet{Ricci_Rewiring} add edges based on Ricci curvature as a preprocessing to reduce the bottleneck in message passing, which results in oversquashing. 
\citet{LDS-GNN} used bi-level programming to learn a generative model that optimizes the adjacency matrix.
\citet{DHGR} add and delete edges based on feature and label neighborhood distribution similarity.
However, they require a large ratio of labeled samples, which are unavailable in common node classification tasks.
\citet{DGM} proposed Differentiable Graph Module (DGM), which learns a latent graph, similar to our model.
However, in contrast to DGM, our model uses the existing edges as a prior, while DGM learns a new adjacency matrix from scratch.
\citet{dpp} and \citet{duan2024layerdiverse} use negative sampling in a stochastic determinant point process to optimize the graph spectrum for the task.
However, it is not end-to-end learnable like our method.
\citet{DropEdge} randomly removes edges during training, which not only alleviates the effect of over-smoothing but also acts as a regularizer to prevent overfitting.
\citet{DRew} enhances the density in a cascading way to obtain a denser graph. 
While many of these methods improve graph expressivity for tasks like graph classification, they overlook the role of neighborhood distribution consistency for node classification in heterophilic graphs.

\section{Preliminaries and Assumptions}
\label{sec:preleminaries}
Let $\gG=(\gV, \gE)$ be a graph, where $\gV$ is the set of nodes and $\gE$ is the set of edges with $|\gV|=n$, $|\gE|=m$.
Each node $v_i \in \gV$ has an associated $d$-dimensional feature vector $x_i \in \mathbb{R}^d$ and a class label $y_i \in \mathcal{C}$. 
The features can be summarized in a matrix $\mX \in \mathbb{R}^{n \times d}$.
The adjacency matrix of $\gG$ is denoted by $\mA(\gG) \in \{0,1\}^{n \times n}$.
The set of nodes with a specific class is defined by $\gV_c = \{ v_i \in \gV \mid y_i = c \}$.
We assume that the features of nodes from the same class are distributed according to the same feature distribution, \ie $x_i \sim \mathcal{F}_{c=y_i}$ and that the neighbors for a node $v_i$ are independent and distributed according to $\mathcal{D}_{i}$, where $\mathcal{D}_{i}$ is the neighborhood distribution of node $v_i$.

We assume that the neighborhood distribution of class $c$, $\mathcal{D}_{c}$, for nodes of a specific class $c \in \mathcal{C}$ is a mixture distribution, \ie consists of multiple components.
Each component is having its own mode, \ie $\mathcal{D}_{c} = \sum_{l=1}^{k_c} \pi_l \mathcal{D}_{l}$, $k_c$ is the number of components for class $c$, and $\pi$ is a $k_c$-dimensional categorical distribution, see Figure~\ref{fig:example}.
In semi-supervised node classification, the goal is to learn a function $f: \gV \rightarrow \mathcal{C}$ based on $\gG$, $\mX$, $\gV_{train}$, and $\gV_{trans}$, where $\gV_{train}, \gV_{trans} \subset \gV$ are labeled subsets and unlabeled subsets of the nodes, respectively.
% The subset $\gV_{train}$ is labeled and used for training, $\gV_{trans}$ is unlabeled but used for training as well.
The goal is to predict the classes of the unlabeled subset $\gV_{test} \subseteq \gV_{trans}$.

%\section{Impact of Neighborhood Distributions in MPNNs for Node Classification}
\section{Neighborhood Distributions in MPNNs}

\label{sec:NeighDistImpact}
We investigate the decomposition of the neighborhood distribution of a class $c$ in $k_c$-many components. 
We aim to understand the neighborhood distribution of each class as a mixture of $k_c$ components, as stated in Section~\ref{sec:preleminaries}.
We show theoretically and empirically that differences in neighborhood distributions are harmful to the MPNN performance. 
The decomposition of neighborhood distributions has an impact on homophily measures in such a way that they fail to provide useful information about a graph's MPNN performance.

So far, prior literature either assumed that MPNNs only work for homophilic graphs~\cite{CPGNN, BlockModellingGCN, HOG-GNN} or graphs where the classes have consistent neighborhood distribution~\cite{HomophilyNecessity}.
\citet{HomophilyNecessity} showed that the distance of an embedding is close to the mean embedding of its class.
We show that this only holds true if the distance between the components of the neighborhood distributions of these nodes is small, since the distance between the expectations of two embeddings is lower bound by the distance of the expectations of their neighborhood distributions.
The assumptions of Section~\ref{sec:preleminaries} allow us to prove the following theorem.
%see Appendix~\ref{appendix:theorem 1}.

\begin{theorem}
    \label{thorem:expected_diff}
    Consider a graph $\gG=(\gV, \gE)$, with class-specific feature distributions $\{\mathcal{F}_c, c \in C\}$, and discrete neighborhood distributions $\{\mathcal{D}_l, l \in [k_c]\})$ for each class, fulfilling the assumptions above.
    Then for two nodes $v_i, v_j \in V$, with the same class $y_i = y_j$ and different discrete neighborhood distributions $\mathcal{D}_{p_i} \neq \mathcal{D}_{p_j}$ the expected distance between their MPNN embeddings $h_i$, $h_j$ is lower bounded by the distance of the means of the neighborhood distribution components:
    
    \begin{center}
    \resizebox{0.99\linewidth}{!}{%
    $\displaystyle
    \mathbb{E}\big[\norm{h_i - h_j}\big] \geq \sigma_{min}(W) \norm{
    \mathbb{E}_{
        x \sim (\mathcal{F}_c | c \sim D_{p_i})}[x]
    -
    \mathbb{E}_{
        x \sim (\mathcal{F}_c | c \sim D_{p_j})}[x]
    }
    $%
    }
    \end{center}

where $\sigma_{min}(W)$ denotes the smallest singular value of the learnable weight matrix $W$.
\end{theorem}

\begin{proof}

By using Jensen's inequality, we get: 
\begin{equation*}
    \mathbb{E}[\norm{h_i - h_j}] \geq \norm{\mathbb{E}[h_i - h_j]} = \norm{\mathbb{E}[h_i] - \mathbb{E}[h_j]} 
\end{equation*}
With the aggregation mechanism of the MPNN, we express the expectation of the embeddings in terms of the expectation of the neighbors and features:
\begin{align*}
& \norm{\mathbb{E}\left[\sum_{k \in N(i)} \frac{1}{\deg(i)} W x_k\right] 
- \mathbb{E}\left[\sum_{l \in N(j)} \frac{1}{\deg(j)} W x_l\right]} \\
=& \left\lVert W \left( \frac{1}{\deg(i)} \sum_{k \in N(i)} \mathbb{E}_{x_k \sim \mathcal{F}_c ,\, c \sim D_{p_i}}[x_k] \right. \right. \\
 &\quad\quad\quad\quad\quad \left. \left. - \frac{1}{\deg(j)} \sum_{l \in N(j)} \mathbb{E}_{x_l \sim \mathcal{F}_c ,\, c \sim D_{p_j}}[x_l] \right) \right\rVert \\
=& \norm{W \left( \mathbb{E}_{x \sim \mathcal{F}_c ,\, c \sim D_{p_i}}[x] 
- \mathbb{E}_{x \sim \mathcal{F}_c ,\, c \sim D_{p_j}}[x] \right)} \\
\geq & \sigma_{\min}(W) \norm{\mathbb{E}_{x \sim \mathcal{F}_c ,\, c \sim D_{p_i}}[x] 
- \mathbb{E}_{x \sim \mathcal{F}_c ,\, c \sim D_{p_j}}[x]}
\end{align*}
Note that $\sigma_{min}$ is the smallest singular value of $W$, and we used the property that the norm of a matrix is always larger than its smallest singular value for the last step.

\end{proof}

Intuitively, it shows that the expected distance of the MPNN embeddings is lower bounded by the distance of the neighborhood distribution, although the nodes are from the same class, \ie share the same feature distribution.
This is a problem since embeddings far from each other tend to be classified into different classes.

%\todo[inline]{Corollary with pairwise distribution?}

We empirically address this problem of multiple neighborhood components per class by splitting classes into subsets based on %distinct
$1$-hop neighborhood distributions.  
Specifically, we compute the empirical $1$-hop distribution, \(\hat{p}_{N(v_i)}\), for each node \(v_i\).  
For each class, we apply a Gaussian Mixture Model (GMM) to cluster these distributions and distinguish between different modes in the neighborhood distribution within the class.  
The number of clusters is determined using the Bayesian Information Criterion (BIC)~\cite{BIC}, which penalizes models with many parameters. %, such as too many clusters.
We perform hard clustering with the GMM, assigning each node to a single pseudo-class, which represents a component of the neighborhood-mixture distribution of a class; details can be found in Appendix~\ref{app:GMM_Clust}.  

Subsequently, we compare the refined $1$-hop neighborhood distributions based on pseudo labels against the original ones using edge homophily~\cite{H2GNN}, adjusted homophily~\cite{LabelInformativeness}, and Label Informativeness (LI)~\cite{LabelInformativeness}.
%, see Appendix~\ref{app:hom} for details on these measures.
%
We compare the measures on commonly used homophilic and heterophilic node classification datasets in Table~\ref{tab:pre_exp_li}.

%\todo[inline]{Add number of classes before and after clustering}
\begin{table}[htbp]
    \centering
    \caption{Measures of neighborhood distribution of original dataset labels and neighborhood-based pseudo-labeled classes after clustering.}
     \begin{adjustbox}{width=\linewidth}
    \begin{tabular}{l|ccc|ccc}
    \toprule
    & \multicolumn{3}{c|}{} & \multicolumn{3}{c}{Neighborhood-based} \\
    & \multicolumn{3}{c|}{Dataset labels} & \multicolumn{3}{c}{ pseudo labels} \\
    \midrule
    Dataset & $h_{edge}$ & $h_{adj}$ & $LI_\text{edge}$  & $h_{edge}$  & $h_{adj}$  & $LI_\text{edge}$ \\
    \midrule
    \multicolumn{7}{l}{Homophilic} \\
    \midrule
     Cora & $0.81$ & $0.77$ & $0.59$ & $0.45$ & $0.41$ & $0.52$ \\
    CiteSeer & $0.74$ & $0.67$ & $0.45$ & $0.42$ & $0.38$ & $0.55$ \\
    PubMed & $0.80$ & $0.69$ & $0.41$ & $0.25$ & $0.19$ & $0.32$ \\
    OGBn-ArXiv & $0.66$ & $0.59$ & $0.46$ & $0.33$ & $0.30$ & $0.36$ \\
    \midrule
    \multicolumn{7}{l}{Heterophilic} \\
    \midrule
    Squirrel & $0.22$ & $0.01$ & $0.00$ & $0.05$ & $-0.04$ & $0.11$ \\
    Chameleon & $0.24$ & $0.04$ & $0.04$ & $0.13$ & $0.03$ & $0.18$ \\
    Actor & $0.22$ & $0.01$ & $0.00$ & $0.03$ & $0.00$ & $0.07$ \\
    Roman-Empire & $0.05$ & $-0.05$ & $0.11$ & $0.00$ & $0.00$ & $0.35$ \\
    Questions & $0.84$ & $0.02$ & $0.00$ & $0.19$ & $0.03$ & $0.03$ \\
    Minesweeper & $0.68$ & $0.01$ & $0.00$ & $0.38$ & $0.13$ & $0.07$ \\
    Tolokers & $0.59$ & $0.09$ & $0.01$ & $0.07$ & $0.04$ & $0.10$ \\
    Amazon-ratings & $0.38$ & $0.14$ & $0.03$ & $0.06$ & $0.04$ & $0.18$ \\
    \bottomrule
    \end{tabular}
    \end{adjustbox}
    
    \label{tab:pre_exp_li}
\end{table}

The results show that LI increases for all heterophilic datasets.
\citet{LabelInformativeness} showed that LI strongly correlates with MPNN performance.
Based on the results of \citet{LabelInformativeness}, splitting up the classes to increase the LI makes the task easier for MPNNs since each class contains a fixed, distinguishable number of neighborhood distributions.
Therefore, our experiment shows that decomposing the neighborhood distribution of a single class in a mixture of multiple neighborhood distributions makes the task easier.
However, in pre-experiments, we observe that learning the decomposed classes obtained by this approach does not improve MPNN classification performance.
The reason is that the aggregated node's features remain unchanged, yielding the same expected embeddings (in the sense of Theorem~\ref{thorem:expected_diff}) despite the updated labels.
We see that the existing measures by~\citet{LabelInformativeness,H2GNN} alone are insufficient to reliably predict MPNN performance based on neighborhood distributions. 
We conclude that the edges must be changed to make the neighborhood distribution more consistent, as explored below.

\section{Gumbel-MPNN: Learnable Graph Rewiring via Gumbel-Softmax}
\label{sec:methods}

We describe our end-to-end differentiable graph rewiring model based on Gumbel Softmax, explain the gradient computation, our regularization terms, and the reduction of the complexity by edge candidate selection. 
We use an edge model $g_u: X,A(G) \rightarrow [0,1]$ with parameters $u$ to estimate the parameters of the probability distributions $\theta_{ij}$ for each (potential) edge $e_{ij}$, where $\bm{\theta} = g_u(X,A(G))$ denote the parameters of the probability distribution to sample a whole adjacency matrix.
Based on these parameters, we sample edges according to $e_{ij} \sim p_{\theta_{ij}=g_u(x_i, x_j)}(E_{ij})$.
The probability for an edge is modeled by a Bernoulli distribution, \ie $p_{\theta_{ij}}(E_{ij} = 1) = \sigmoid(\theta_{ij})$, where $E_{ij}$ is a random variable modeling the presence of edge $e_{ij}$ from $v_i$ to $v_j$.
The resulting probability to sample a specific new adjacency matrix is $\tilde{A}(\tilde{G}) \sim p_{\bm{\theta}} = \prod_{i,j=1}^n p_{\theta_{ij}}(E_{ij} = 1)$. 
The features $X$ and the rewired adjacency matrix $\tilde{A}$ are then input for the MPNN $f_w: \mathbb{R}^{n \times d} \times \tilde{A}(\tilde{G}) \rightarrow \mathcal{C}$ with parameters $w$ to compute the final predictions.

\textbf{Gradient Computation.}
The discrete sampling of edges from a Bernoulli distribution is not differentiable. 
Therefore, we need to estimate the gradient.
Since the edges are independent and Bernoulli distributed, we can use the Gumbel-Softmax reparameterization trick~\cite{gumbel} to estimate the gradients. 
This requires selecting a temperature parameter $\tau$, which determines the degree of determinism in the edge sampling. 
We used a small value of $\tau=0.1$, which corresponds to very deterministic sampling.
Let $l(\hat{y}, y) = l(f_w(\tilde{A}, X), y)$ be the loss function for a single sample with a sampled adjacency matrix.
Therefore, the overall loss function is given by:
\begin{equation}
    \mathcal{L}(A(G), X, y;w,u) = \mathbb{E}_{\tilde{A} \sim p_{\mathbf{ \bm{\theta}}}}[\ell(f_w (X, \tilde{A}), y)]
\end{equation}
where $\bm{\theta} = g_u(X,A(G))$.
The gradient for $w$, the parameters of $f$, can be computed by regular backpropagation. 
The parameters $u$ of the edge model $g_u$ require drawing Monte-Carlo samples from $p_{\bm{\theta}}$, which is expensive and does converge slowly~\cite{reinforce}.
Fortunately, $p_{\bm{\theta}}$ is Bernoulli distributed as a product of Bernoulli distributed variables. Therefore, we can use the Gumbel-Softmax trick to estimate the gradient effectively~\cite{gumbel}.

For the final model, we use a GCN as $f_w$ and a Bilinear-MLP for the edge model $g_u$, \ie a model of the form $g_u(x_i, x_j) = \theta_{ij} = \sigmoid(x_i W x_j + b)$ for each potential edge. 
%More details can be found in Appendix~\ref{app:details_gumbel}.

\textbf{Regularization.}
By adding different structural regularization terms, we can encourage the model to learn different neighborhood distributions. 
The regularization terms can be part of hyperparameter tuning or decided upon based on task knowledge.

%degree reg
First, we introduce a degree regularizer that encourages each node to have a minimum degree.
It is well known that MPNNs struggle with nodes with a low degree~\cite{high_deg_good}.
Therefore, this regularization term focuses on these nodes to increase their degree and, thus, the MPNN classification performance.
The regularizer is given by
%\todo[inline]{Could this be generalized to any degree distribution with min degree as a special case?}
\begin{equation}
    \mathcal{L}_{deg} = \frac{1}{|V|}\sum_{v_i \in V} 
    \operatorname{ReLU}\left(d^* - d_i + \delta\right)^2 \,,
\label{eq:struct_reg}
\end{equation}
where $d_i$ is the degree of node $i$, $d^*$ is the targeted minimum degree, and $\delta$ is a tolerance level.
The $ReLU$ ensures that the regularizer ignores nodes with a high degree. 
%Could also be used to match specific strucutres or distirbutions

%label consistency
The label consistency regularizer promotes new edges between nodes with the same label to increase the homophily of the graph.
The regularizer is defined by

\begin{equation}
\mathcal{L}_{label} = \frac{1}{|\tilde{E}|} \sum_{e_{ij} \in \tilde{E}} (1-\tilde{y_i}~\tilde{y_j}),
\end{equation}
where $\tilde{E}$ is the set of edges in the rewired graph, and $\tilde{y_i}$ is the true label if $v_i \in V_{train}$.
Otherwise, it is the model prediction $\hat{y_i}$. 

%one hot consistency loss
The loss $\mathcal{L}_{label}$ encourages neighbors to have the same label, \ie is based on homophily.
The neighborhood consistency loss $\mathcal{L}_{Ncon}$ extends this by encouraging nodes with the same $1$-hop label distribution to be connected, independent of whether the neighbors are homophilic or not. 
The $1$-hot neighborhood consistency regularizer is defined by

\begin{equation}
    \mathcal{L}_{Ncon} = \frac{1}{|\tilde{E}|} \sum_{e_{ij} \in \tilde{E}} (1-\tilde{p_i} \tilde{p_j}) \,, 
\end{equation}
where $\tilde{p_i}$ is the $1$-hop aggregation of the predictions or true labels, where available, including self-loops.
This regularization reduces the number of components in a class's mixture neighborhood distribution.

The $1$-hop neighborhood consistency loss encourages connecting nodes with the same neighborhood distribution. 
However, it is also important that nodes of different classes have dissimilar neighborhood distributions.
This allows better separation of nodes of different classes.
For this reason, we introduce an inter-class distance loss defined as

\begin{equation}
    \mathcal{L}_{inter} = \frac{1}{|C|} \sum_{c, \tilde{c} \in C, c \neq \tilde{c}} ReLU(m - \| r_c - r_{\tilde{c}} \|_2) \,,
\end{equation}
where $r_c$ is a neighborhood distribution prototype computed by $p_k = \frac{1}{|\mathcal{V}_c|}\sum_{v \in \mathcal{V}_c} w_v \tilde{y}_v$, $m$ is some margin, and $\tilde{y}_v$ is the one-hot encoded prediction of the label.
The weight $w_v$ is the product of the normalized node degree of $v$ and the normalized entropy over $\tilde{y}_v$.
Again, the label is the true label if $v \in V_{train}$. Otherwise, it is the model prediction.

\textbf{Complexity and Edge Candidate Selection.}
The complexity of our method depends on the complexity of the MPNN $f$.
We denote the complexity by $\mathcal{O}(\operatorname{MPNN}(|\gV|,|\gE_{\text{cand}}|,h,d,L,C)$, where $|\gE_{\text{cand}}|$ is the maximal number of candidate edges that can be sampled in one forward pass, $h$ is the hidden dimension, $L$ is the number of layers, and $C = |\mathcal{C}|$ is the number of classes. 
Except for $h$, each parameter is linear in the complexity for most MPNNs, \eg GCN~\cite{GCN}.
The complexity of the edge model $g$ is given by $\mathcal{O}(h^2, |\gE_{\text{cand}}|)$.
The maximal possible number of edges for $|\gE_{\text{cand}}|$ is $|\gV|^2$ by considering every possible edge as a candidate.

Since $\mathcal{O}(|\gV|^2)$ is not feasible for large graphs, we limit the number of potential edges.
For this reason, we select a subset of potential edge candidates $\gE_{\text{cand}} \subseteq \gV \times \gV$.

We use one of four edge candidate pre-selection strategies.
The first one is based on feature similarity per node.
For each node, we select the $s$-many nodes with the highest feature similarity based on the dot product as potential edge candidates. 
Therefore, the $\mathcal{O}(|V|^2)$ factor reduces to $\mathcal{O}(s|V|)$, where $s$ is a small number, \eg three or five.
The second strategy is based on the highest node feature similarities in the whole graph.
Based on all node pairs in the graph, we select the $s$-many edge candidates with the highest feature similarity.
In this setting, it is possible that some nodes do not have any edge candidates connected to them. 
This can be a disadvantage for some tasks.
The complexity reduces from $\mathcal{O}(|V|^2)$ to $\mathcal{O}(s)$, here $s$ is larger, \eg $2|E|$.
Higher is, in general, better; the limitation is the available memory.
The third strategy leverages the local structure in the graph.
We randomly select $s$-many nodes in the $2$-hop neighborhood of each node.
This results in $\mathcal{O}(|V|^2)$ many edges when using the whole adjacency matrix, reducing to $\mathcal{O}(s|V|)$ many edges after the edge candidate selection.

\section{Experiments}
\label{sec:experiments}

We show the effectiveness of Gumbel-MPNN by evaluating it against baselines on homophilic and heterophilic node classification datasets.
We demonstrate that  Gumbel-MPNN learns structural properties like average degree and that it captures long-range dependencies by adding the crucial edges for the task, \eg introducing shortcut edges.
All node classification experiments are conducted in a transductive setting.
%
%Details on the datasets and the hyperparameters can be found in the supplementary material.
For details on the datasets, see Appendix~\ref{app:dataset}, and for the hyperparameters see Appendix~\ref{app:hyperparam}.

\textbf{Node Classification on Benchmark Datasets.}
We compare Gumbel-MPNN against standard MPNNs to demonstrate the benefit of rewiring.
As baseline models, we consider an optimized GCN~\cite{GCN}, GraphSAGE~\cite{GraphSage}, and GAT~\cite{GAT}.
It has been shown that optimizing these classical models results in state-of-the-art classification performance~\cite{HomophilyNecessity, strong_baselines}.
Additionally, we compare our model to an MLP, \ie a model that does not use the edges.
An MLP performs well when the classes are well determined by the node features. 
However, it can not be corrupted in cases where edges may be uninformative or even harmful.
To show that the rewiring is not just random regularization, we also compare it to DropEdge~\cite{DropEdge}, which randomly drops edges during training.
Finally, we compare our model to a rewiring variant of a GCN~\cite{GCN} with the Stochastic Discrete Ricci Flow (SDRF)~\cite{Ricci_Rewiring}.
It rewires the graph in a pre-processing step by adding edges before the training.
The goal is to reduce bottlenecks in the graph and to evaluate the benefit of learning the graph structure.

%Datasets
We evaluate all models on the homophilic datasets 
Cora, 
CiteSeer~\cite{CoraCiteSeer}, 
PubMed~\cite{PubMed}, and 
OGBn-ArXiv~\cite{ogb-benchmarks}, and
the heterophilic datasets Squirrel~\cite{SquirrelChameleon}, Chameleon~\cite{SquirrelChameleon} (where we use the filtered version from \citet{HeteroPhilDSNew}), Actor~\cite{GeomGNN}, and the datasets Roman-Empire, Questions, Tolokers, Minesweeper, and Amazon-Ratings~\cite{HeteroPhilDSNew}.

%Measures
We measure the node classification performance and compare the initial neighborhood distribution of the graph to the neighborhood distribution of the rewired graphs.
To this end, we compute the standard deviation per class and average it, as well as the number of classes detected by a Gaussian Mixture Model based on the Bayesian Information Criterion. %described in Appendix~\ref{app:GMM_Clust}. 

The results can be found in Table~\ref{tab:results_real-world} and Table~\ref{tab:distribution_comparison}.
We see that Gumbel-MPNN is on par with the baselines and reduces the neighborhood deviation and number of components per class effectivly.

\newcommand{\mytextsubscript}[1]{}
\renewcommand{\mytextsubscript}[1]{{\color{black}~\textsubscript{#1}}}

\begin{table*}[!htbp]
    \centering
    %\caption{Measures of neighborhood distributions on the global graph and the relabeled classes after \textit{clustering on the entire dataset (100\% of the labels}, aka train/val/test splits). Thus, this setup simulates a ``perfect'' splitting of the classes. 
    %For training the models, however, we use the train/val splits of the dataset.}
    \caption{Node classification performance on the homophilic and heterophilic real-world datasets. The best result is marked in bold. All methods were tuned per dataset with the same hyperparameter space, \eg between $1$ and $5$ layers.}

    \begin{adjustbox}{width=\textwidth}

    \begin{tabular}{l|cccccccccccc}
    
    \toprule 
                   & Cora & CiteSeer & PubMed & OGBn-ArXiv &Squirrel & Chameleon & Actor & Roman-Empire & Questions & Minesweeper & Tolokers & Amazon-ratings\\
    \midrule
     MLP  & $75.61\mytextsubscript{0.20}$ & $72.89\mytextsubscript{0.10}$ & $87.30\mytextsubscript{0.04}$ &$55.06\mytextsubscript{0.09}$ & $39.95\mytextsubscript{0.23}$ & $37.72\mytextsubscript{0.62}$ & $36.50\mytextsubscript{0.13}$ & $65.88\mytextsubscript{0.04}$ & $70.82\mytextsubscript{0.10}$ & $50.38\mytextsubscript{0.12}$ & $73.62\mytextsubscript{0.20}$ & $46.68\mytextsubscript{0.13}$\\
     GCN & $ 87.47\mytextsubscript{0.12}$ & $76.10\mytextsubscript{0.15}$ & $88.07\mytextsubscript{0.04}$ & $71.27\mytextsubscript{0.17}$ & $35.49\mytextsubscript{0.16}$  & $39.17\mytextsubscript{0.31}$ & $35.21\mytextsubscript{0.07}$ & $78.82\mytextsubscript{0.09}$ & $76.07\mytextsubscript{0.24}$ & $91.50\mytextsubscript{0.06}$ & $80.75\mytextsubscript{0.12}$ & $49.65\mytextsubscript{0.13}$ \\
     GraphSAGE & $87.89\mytextsubscript{0.13}$ & $76.12\mytextsubscript{0.14} $ & $88.22\mytextsubscript{0.04}$ & $69.65\mytextsubscript{0.27}$ & $37.22\mytextsubscript{0.22}$  & $41.79\mytextsubscript{0.40}$ & $36.79\mytextsubscript{0.10}$ & $81.78\mytextsubscript{0.06}$ & $75.38\mytextsubscript{0.09}$ & $\textbf{93.58}\mytextsubscript{0.05}$ & $80.16\mytextsubscript{0.12}$ & $50.61\mytextsubscript{0.10}$ \\
     GAT & $87.49\mytextsubscript{0.15}$ & $\textbf{76.43}\mytextsubscript{0.10}$ & $88.51\mytextsubscript{0.04}$ & $\textbf{71.97}\mytextsubscript{0.12}$ & $\textbf{37.41}\mytextsubscript{0.15}$  & $40.23\mytextsubscript{0.33}$  & $34.91\mytextsubscript{0.12}$ & $81.45\mytextsubscript{0.09}$ & $72.54\mytextsubscript{0.10}$ & $92.38\mytextsubscript{0.09}$ & $80.99\mytextsubscript{0.17}$ & $\textbf{50.56}\mytextsubscript{0.06}$\\
     GCN with DropEdge  & $87.93\mytextsubscript{0.13}$ & $75.99\mytextsubscript{0.15}$ & $88.64\mytextsubscript{0.03}$ & $71.75\mytextsubscript{0.16}$ & $34.96\mytextsubscript{0.23}$ &  $40.76\mytextsubscript{0.41}$ & $33.59\mytextsubscript{0.10}$ & $78.26\mytextsubscript{0.11}$ & $55.92\mytextsubscript{0.11}$ & $89.95\mytextsubscript{0.07}$ & $81.49\mytextsubscript{0.11}$  & $49.84\mytextsubscript{0.13}$ \\
     SDRF with GCN  & $87.88\mytextsubscript{0.13}$ & $76.12\mytextsubscript{0.14}$ & $88.05\mytextsubscript{0.04}$ & OOT & $36.07\mytextsubscript{0.20}$  & $35.88\mytextsubscript{0.15}$ & $35.15\mytextsubscript{0.05}$ & $80.60\mytextsubscript{0.07}$ & $57.16\mytextsubscript{0.14}$ & $89.93\mytextsubscript{0.06}$ & $80.90\mytextsubscript{0.11}$ & $44.03\mytextsubscript{0.18}$
 \\
     \hline
     \textit{Gumbel-MPNN (ours)}  & $\textbf{87.95}\mytextsubscript{0.10}$  & $75.54\mytextsubscript{0.17}$ & $\textbf{88.73}\mytextsubscript{0.04}$ & $70.44\mytextsubscript{0.19}$ & $36.93\mytextsubscript{0.21}$ & $\textbf{41.96}\mytextsubscript{0.32}$  & $\textbf{37.13}\mytextsubscript{0.10}$ & $\textbf{82.23}\mytextsubscript{0.07}$
 & $\textbf{76.61}\mytextsubscript{0.12}$  & $91.18\mytextsubscript{0.13}$  & $\textbf{81.56}\mytextsubscript{0.41}$  & $50.08\mytextsubscript{0.06}$ \\       
     
     \bottomrule
    \end{tabular}
    \end{adjustbox}
    
    \label{tab:results_real-world}
\end{table*}

\begin{table}[!htbp]
    \centering
    \caption{Average standard deviation of the neighborhood distribution and number of clusters with a GMM before and after training.}
     \begin{adjustbox}{width=0.9\linewidth}
    \begin{tabular}{l|rr|rr}
    \toprule
    & \multicolumn{2}{c|}{Original Graph} & \multicolumn{2}{c}{Rewired Graph} \\
    \midrule
    Dataset & \makecell{Std. per\\ Class $\downarrow$} & \makecell{Number of\\ clusters $\downarrow$} & \makecell{Std. per\\ Class $\downarrow$} & \makecell{Number of\\ clusters $\downarrow$} \\
    \midrule
    Cora & 0.0957 & $47$ & 0.0853 & $41$  \\
    CiteSeer & 0.1396 & $74$ & 0.1340 &  $69$ \\
    PubMed & 0.1700 & $59$ & 0.1659 &  $31$ \\
    OGBn-ArXiv & 0.0450 & $377$ &  0.0450 &  $377$\\
    \midrule
    Squirrel & 0.1433 & $94$ & 0.1338 &  $61$ \\
    Chameleon & 0.1675 & $69$ &  0.1573 & $59$ \\
    Actor & 0.1686 & $109$ & 0.1707 & $95$ \\
    Roman-Empire & 0.0790 & $266$ & 0.0707 & $228$ \\
    Questions & 0.1691 & $36$ & 0.1591 & $33$ \\
    Minesweeper & 0.1299 & $18$ & 0.1439 &  $44$\\
    Tolokers & 0.1995 & $35$ & 0.1891 & $30$ \\
    Amazon-ratings & 0.1633 & $105$ & 0.1534  & $103$ \\
    \bottomrule
    \end{tabular}
    \end{adjustbox}
    
    \label{tab:distribution_comparison}
\end{table}

\textbf{Synthetically Adapted Neighborhood Distributions.} 
We investigate the model behavior with multiple discrete neighborhood distributions.
In contrast to Section~\ref{sec:NeighDistImpact}, we do not divide the nodes based on their neighborhoods but create these neighborhoods artificially. 
To prepare the dataset for this experiment, we use the nodes, features, and labels of the existing real-world datasets but replace the edges with a new edge set. 
To create the edges, we split each class into $k$-many neighborhood distributions.
For each node $v_i \in V$ with class $y_i$, we sample a predefined distribution $D_{p_l}$ with $l \in \mathcal{U}({1,2, \dots k})$.
Each $D_{p_l}$ is a different heterophilic neighborhood distribution, with equal probability for the classes.
On this basis, we sample a class $c \sim D_{p_l}$, and from that class, a node $v_j \in V_c$ and add the edge $e_{ij}$ to our graph.
We train the models for $k \in \{1, \dots 7\}$, where 1 means each class has exactly one neighborhood distribution component. 
See Figure~\ref{fig:synthetic_result} for the result.

\begin{figure}[!htbp]
    \centering
    \includegraphics[width=0.75\linewidth]{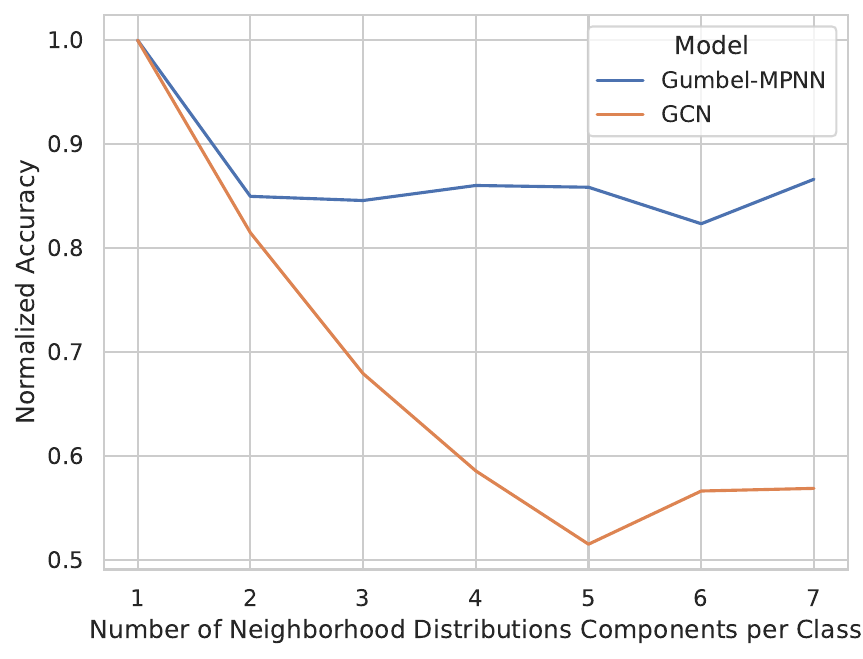}, which
    \caption{GCN versus Gumbel-MPNN on Actor with different numbers of synthetically generated neighborhood distributions.}
    \label{fig:synthetic_result}
\end{figure}

\textbf{Adapting the Graph Structure from a Supervision Signal.}
Our approach can adapt a graph to given graph properties. 
We define a target degree $d^*$ as the average degree of the graph plus $5$ and learn it by using the loss of Equation~\ref{eq:struct_reg} with a high weight and the respective degree.
We compare the min, average, and max degrees before and after training to our specified degree to show that the model adapted the graph more towards the desired structure.
The result is presented in Table~\ref{tab:struct_learn}, showing that Gumbel-MPNN reduces the deviation in a class's neighborhood on $10$ out of $12$ datasets.

\begin{table}[!ht]
    \centering
     \caption{Min, max, and average degrees of the original and rewired graph before and after the structure learning experiment.}
    \begin{adjustbox}{width=0.9\linewidth}
    \begin{tabular}{l|rrr|rrr}
    \toprule
     & \multicolumn{3}{c|}{Original Graph} & \multicolumn{3}{c}{Rewired Graph} \\
     \midrule
     Dataset  &  Min & Avg & Max &  Min & Avg &  Max\\
     \midrule
      Cora   & 1 & 3.9 & 168 & 4 & 11.5 & 161\\ 
      CiteSeer  & 0 & 2.7 & 99 & 5 & 12.9 &  22\\
      PubMed  & 0 & 4.5 & 171 & 1 & 13.8 & 508\\
      OGBn-ArXiv  & 1 & 13.67 & 13\,161 & 1 & 13.68 & 13\,161\\
      Squirrel & 1 & 76.3 & 1\,904 & 1 & 76.5 & 1\,796 \\
      Chameleon  & 1 & 27.6 & 732 & 1 & 28.4 & 686\\
      Actor  & 1 & 7.0 & 1303 & 1 & 9.5 & 1\,284\\
      Roman-Empire  & 2 & 2.9 & 14 & 1 & 4.6 & 2\,703 \\
      Questions  & 1 & 6.3 & 1\,539 & 4 & 16.2 & 133\\
      Minesweeper  & 3 & 7.9 & 8 & 4 & 16.2 & 133 \\
      Tolokers  & 1 & 88.3 & 2\,138 & 0 & 28.3 & 637 \\
      Amazon-ratings  & 5 & 7.6  & 132  & 7 & 10.5 & 130 \\
    \bottomrule
    \end{tabular}
    \end{adjustbox}

    \label{tab:struct_learn}
\end{table}

\textbf{Long-range Dependencies.}
We use the synthetic dataset of \citet{ProbRewiring} to evaluate the capabilities of the model to learn long-range dependencies in node classification.
This dataset consists of $2^R$ binary trees, each with $2^R$ leaves with a $0$ or $1$ feature.
The task is to predict the label at the root node, which is the sum of the features at the leaf nodes, \ie the number of leaves with $1$ as a feature.
For a model to perform well on this task, it needs to propagate the features of the leaves across all levels of the tree up to the root node.
See Figure~\ref{fig:sample_leafcount} for an example, where the blue nodes are $1$ and the orange nodes are $0$, therefore, the model needs to predict a $3$ here.
We compare our model to a regular GCN~\cite{GCN} with two and four layers and Stochastic Discrete Ricci Flow (SDRF)~\cite{Ricci_Rewiring} to show that the dependency needs to be learned to obtain good performance.
The results can be found in Table~\ref{tab:long_range}.

\begin{figure}[!htbp]
    \centering
    \begin{minipage}{0.24\textwidth}
        \centering
    \begin{adjustbox}{width=\linewidth}

    \begin{tabular}{l|r}
    \toprule
     Model & Accuracy \\
     \midrule
     2-layer GCN    &  $14.29$\\
     4-layer GCN     &  $100.00$\\
     2-layer Gumbel-MPNN  & $100.00$ \\
     \bottomrule
    \end{tabular}
    \end{adjustbox}
        \captionof{table}{Accuracy of on the long-range dependency task.}
        \label{tab:long_range}
    \end{minipage}%

    \begin{minipage}{0.23\textwidth}
        \centering
        \includegraphics[width=0.9\linewidth]{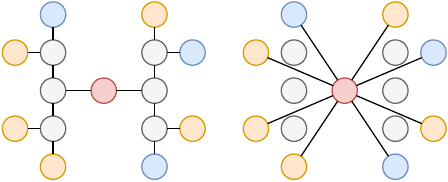}  
        \caption{
        Left: Example input graph; Right: After Gumbel-MPNN Rewiring}
        \label{fig:sample_leafcount}
    \end{minipage}
\end{figure}

\textbf{Oversquashing.}
A common issue with MPNNs is the oversquashing~\cite{Ricci_Rewiring}, \ie the problem when information from distant nodes is compressed into fixed-size embeddings through successive message-passing layers, leading to a bottleneck that inhibits the network from learning meaningful embeddings~\cite{Ricci_Rewiring}.
Common approaches to tackle this problem are Jumping Knowledge or Skip Connection~\cite{JumpingKnowledge, GCNII}, 
and graph rewiring methods~\cite{Ricci_Rewiring, ProbRewiring}. 
However, most rewiring methods are developed for graph-level tasks~\cite{ProbRewiring}, \ie do not scale or are not applicable for node-level tasks.
We propose a new node classification dataset where the model has to learn the relationship of distant nodes through a bottleneck.
Each node on one side of the bottleneck has a corresponding node on the other side, mapped by a one-hot encoded ID. 
The nodes on one side have the label for the node on the other side as a feature, so the model has to propagate the feature through the bottleneck. 
We compare GCN and SDRF as shown in Figure~\ref{fig:oversquashing}.

\begin{figure}[!ht]
    \centering
    \includegraphics[width=0.35\textwidth]{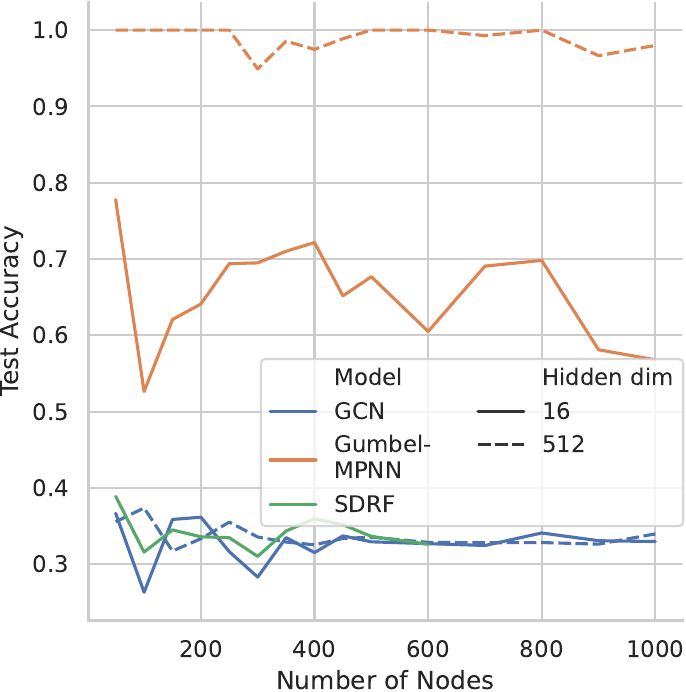}
    \caption{Gumbel-MPNN versus GCN for oversquashing for different values of nodes and hidden dimension.}
    \label{fig:oversquashing}
\end{figure}

\textbf{Edge De-Noising.}
We evaluate the robustness of our model to edge noise. 
The model is trained without artificial edge noise and evaluated on a test set with multiple levels of artificial edge noise.
The noise is added by uniformly sampling a test node $v_i$ and a class $\tilde{c} \sim \mathcal{U}(\{1, \dots,|C|\})$.
We sample a second node $v_j \sim V_{\tilde{c}}$ and add the edge $e_{ij}=(v_i, v_j)$.
This process is repeated for $k \in \{100, 500, 1~000, 10~000, 50~000 \}$ times to investigate multiple noise levels.
We compare our model to a regular GCN, GAT, and GCN with DropEdge. 
Figure~\ref{fig:edgenoise} shows the results on Cora.

\begin{figure}[!ht]
    \centering
    \includegraphics[width=0.4\textwidth]{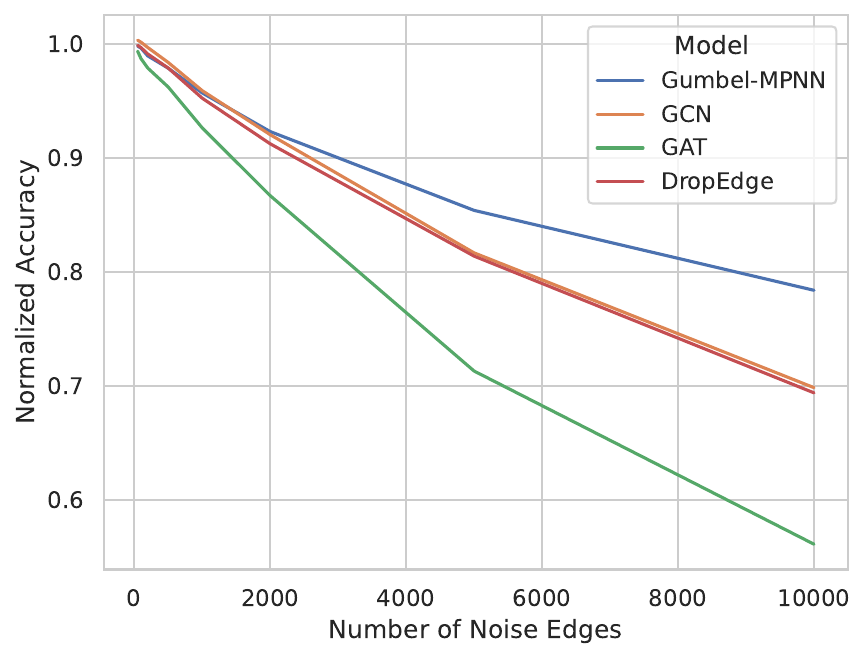}
    \caption{Relative test accuracy depending on noise edges on Cora.}
    \label{fig:edgenoise}
\end{figure}

\section{Discussion}
\label{sec:discussion}

We have demonstrated that our Gumbel-MPNN excels in five challenging tasks for standard MPNNs, including long-range dependencies, edge de-noising, and oversquashing. 
Notably, on node classification benchmarks, Gumbel-MPNN outperforms tuned baselines on 7 out of 12 datasets, validating its effectiveness

\textbf{Node Classification on Benchmark Datasets.}
We have seen that Gumbel-MPNN matches or exceeds the performance of its baselines, confirming that neighborhood consistency is essential for MPNN performance. 
Specifically, Gumbel-MPNN reduces the number of different components in the neighborhood mixture distribution on all datasets.
The only exception is Minesweeper, where grid-based edges lack noise and node features are label-independent, violating the assumption $x_i \sim \mathcal{F}_{c = y_i}$. 
While Gumbel-MPNN harmonizes neighborhood distributions, it can not learn relationships that are absent, such as node features that are independent of labels.
Our results support prior findings that %-tuned standard 
MPNN models can perform well on heterophilic datasets~\cite{HomophilyNecessity, strong_baselines}, regardless of architectural specifics.
%Discuss modifiyng edges in questions makes perofrmance worse - but our model can handle it

\textbf{Synthetic Neighborhoods.}
Standard MPNNs struggle with many components in the neighborhood class distribution (c.f. GCN in Figure~\ref{fig:synthetic_result}).
In contrast, we have observed that Gumbel-MPNN can maintain its performance, showing that our proposed rewiring strategy works as expected.
The steepest decrease for both models is from one to two components per class, as supported by Theorem~\ref{thorem:expected_diff}.
% For six and seven components, the performance goes up a litte, because the distance between the distributions reduced a little.

\textbf{Adapting Graph Structure from Supervision Signal.}
Our results show that the model successfully adapts the graph structure to fit the new average degree in 9 out of 12 datasets.
However, for Tolokers, the degree decreased, and in some cases, the average degree change was either smaller or larger than the anticipated plus five on the original average degree. 
%ce loss > deg reg
A possible cause for his behavior is that the model must still satisfy the classification objective, which can sometimes take priority over the degree objective. 
%bilin mlp
Additionally, the edge model $g_u$, implemented as a Bilinear-MLP, is limited to learning a similarity metric, which may not be sufficiently expressive to capture the required diversity of neighborhood distributions.

\textbf{Long-range Dependencies.}
Through rewiring, Gumbel-MPNN tackles long-range dependency tasks without increasing the number of layers. Specifically, the long-range dependency task is solvable for any model with a sufficient number of layers. For instance, a 4-layer GCN attains a perfect score. 
However, a 2-layer GCN fails to solve the task, whereas a 2-layer Gumbel-MPNN still perfectly solves it. 
Since increasing depth in MPNNs comes with other challenges, \eg oversmoothing~\cite{GCNII}, this is an important advantage of Gumbel-MPNN in settings where long-range dependencies must be considered.

\textbf{Oversquashing.}
Regarding the bottleneck task, we observe that Gumbel-MPNN can solve the oversquashing problem perfectly with a sufficiently high hidden dimension by directly connecting the respective nodes on both sides of the bottleneck.
In contrast, Standard MPNNs cannot encode and bypass the bottleneck for any hidden size.
This shows that disadvantageous connectivity patterns can be alleviated through Gumbel Softmax rewiring.
SDRF can not overcome the bottleneck since it only works on the graph structure, but the task requires considering node feature similarity to connect the right edges.
% Although the a hidden dimension of $512$ should be able to capture at least for small number of nodes GCN does not encode and decode it properly.

\textbf{Edge De-Noising.}
Test-time edge noise is a crucial problem for all regular MPNN models. Yet, our results show that Gumbel-MPNN is more robust against edge noise than other models.
% Our results show that althoug all models are suffering from the test-time edge-noise Gumbel-MPNN is more robust than regular MPNN models.
Given DropEdge's weak performance, we can conclude that randomly removing edges during training does not help the model deal with edge noise at test time, likely because DropEdge does not add any new edges.
Gumbel-MPNN, in contrast, is able to alleviate the issue by learning to ignore the noise edges.

\textbf{Ablation Study}
We conduct an ablation study on the regularization terms of the loss function.
As discussed in Section~\ref{sec:methods}, we only selected one regularization term per dataset based on pre-experiments. 
For the ablation study, we reduce the loss function to the cross-entropy loss, \ie disregard the regularization term.
The results are presented in Table~\ref{tab:ablation}.
We see that the regularization term can improve the result by a large margin, \eg $7$ points on Chameleon, but never decreases it. 

\begin{table}[!ht]
    \centering
    \caption{Ablation study of the regularization for Gumbel-MPNN.}
    \label{tab:ablation}
     \begin{adjustbox}{width=0.4\textwidth}
    \begin{tabular}{l|rr}
    \toprule
%      Dataset &  G.-MPNN w.o. Reg. &  Gumbel-MPNN \\
      Dataset &  Gumbel-MPNN  &  Gumbel-MPNN \\
              &  w/o Regularization &  \\
    \midrule
    Cora & $87.91_{0.13}$ & $87.95_{0.10}$ \\
    CiteSeer & $75.55_{0.17}$ & $75.54_{0.17}$ \\
    PubMed & $88.77_{0.04}$ & $88.73_{0.04}$ \\
    OGBn-ArXiv & $70.20_{0.12}$ & $70.44_{0.19}$ \\
    Squirrel & $54.43_{0.10}$ & $54.47_{0.12}$ \\
    Chameleon & $57.26_{0.71}$ & $64.49_{0.25}$ \\
    Actor & $34.92_{0.11}$ & $37.13_{0.10}$ \\
    Roman-Empire & $77.28_{0.28}$ & $82.23_{0.07}$ \\
    Questions & $75.02_{0.27}$ & $76.61_{0.12}$ \\
    Minesweeper & $90.15_{0.26}$ & $91.18_{0.13}$ \\
    Tolokers & $81.41_{0.39}$ & $81.56_{0.41}$ \\
    Amazon-ratings & $49.93_{0.04}$ & $50.08_{0.06}$ \\
    \bottomrule
    \end{tabular}
    \end{adjustbox}
\end{table}

\section{Limitations}
% x ~ Dc, theorem
Our work assumes that each class follows a specific feature distribution. 
While this assumption is crucial, it is also the underlying assumption of most machine learning models.
%Pairwise neighbors
Gumbel-Softmax rewiring relies on estimating pairwise probabilities for edges, guided by the overall supervision signal. 
This can lead to more consistent neighborhood distributions, as our results show.
However, a model that considers the whole neighborhood of a node, \ie all edges of a node's neighborhood at once instead of one edge at a time, could result in more accurate estimates for the probability of adding or removing edges. 
%heterophilic baselines?
We used a representative set of benchmark datasets, both those with high homophily and high heterophily, and strong baseline models.
We could have compared to models particularly tailored to heterophilic datasets. 
However, following \citet{HomophilyNecessity, strong_baselines},
MPNNs are on par with specialized heterophilic models if tuned properly, as we have also done here.
Generally, although there might be models performing best for specific challenges, \eg robustness to edge noise~\cite{edge_noise_model}, these models fall short on other aspects that we cover, such as capturing long-range dependencies, reducing oversquashing, or handling multiple neighborhood distributions per class.
Thus, overall, we aimed for a solution that addresses those challenges in one model rather than providing a specific solution to a specific aspect.

\section{Conclusion}
\label{sec:conclusion}

Our experiments reinforce the importance of neighborhood consistency.
We show that the number and distance of components in a neighborhood distribution are crucial for MPNN performance. 
%We demonstrate that modifying graph structure can enhance node classification. 
By introducing a Gumbel-Softmax-based rewiring method, we %effectively 
reduce deviations in neighborhood distributions, mitigate oversquashing, and handle long-range dependencies. 
These findings highlight the significance of adaptive graph structures for MPNNs in node classification.
Future work includes investigating the relationship between rewiring and graph spectra, as well as expanding experiments to graph classification and other domains.
Future work could also further investigate the effect of the temperature parameter in Gumbel softmax sampling for graph rewiring.

\bibliography{ecai_bib}

\newpage

%%%%%%%%%%%%%%%%%%%%%%%%%%%%%%%%%%%%%%%%%%%%%%%%%%%%%%%%%%%%%%%%%%%%%%%%%%%%%%%
%%%%%%%%%%%%%%%%%%%%%%%%%%%%%%%%%%%%%%%%%%%%%%%%%%%%%%%%%%%%%%%%%%%%%%%%%%%%%%%
% APPENDIX
%%%%%%%%%%%%%%%%%%%%%%%%%%%%%%%%%%%%%%%%%%%%%%%%%%%%%%%%%%%%%%%%%%%%%%%%%%%%%%%
%%%%%%%%%%%%%%%%%%%%%%%%%%%%%%%%%%%%%%%%%%%%%%%%%%%%%%%%%%%%%%%%%%%%%%%%%%%%%%%
\newpage
\appendix
\onecolumn

\section*{Supplementary Materials}

\section{Theoretical Analysis of Multiple Neighboorhood Distributions}
\label{appendix:theorem 1}

\textbf{Assumptions.}
Given a graph $\gG=(\gV,\gE)$, each node $v_i \in \gV$, has a feature vector $x_i \in  \mathbb{R}^d$ and a label $y_i \in \mathcal{C}$.
We assume that the features of nodes from the same class $c$ are sampled from the same distribution, \ie $x_i \sim \mathcal{F}_{c=y_i}$.
Furthermore, we assume that the neighbors for a node $v_i$ are sampled independently from each other according to some distribution $\mathcal{D}_{l}$, where $l$ is the neighborhood component of the mixture distribution of node $v_i$. % equal or More than |C| Dpi, equal or less than |V| Dpi,

\begin{lemma}[Jensen's inequality]

Given a random variable $X$ and a convex function $\varphi$, then

\begin{equation}
    \varphi(\mathbb{E}[X]) \leq \mathbb{E}[\varphi(X)]
\end{equation}

\end{lemma}

\begin{lemma}

Given two matrices $A$ and $B$, the norm of the matrix product can be lower bounded by

\begin{equation}
    \norm{AB}_2 \geq \sigma_{min}(A)\norm{B}_2
\end{equation}
where $\sigma_{min}(A)$ is the minimal singular value of $A$.
    
\end{lemma}

%\todo[inline]{Can we express this as some distance of distributions? Kind of locks like an MMD, so something like MMD(D1, D2) would be cool. Or just define same variable as these expectations to make it look nicer.
%Nice MMD description: \url{https://stats.stackexchange.com/questions/276497/maximum-mean-discrepancy-distance-distribution}
%}

\begin{theorem}
    \label{app:thorem:expected_diff}
    Consider a graph $\gG=(\gV, \gE)$, with class-specific feature distributions $\{\mathcal{F}_c, c \in C\}$, and discrete neighborhood distributions $\{\mathcal{D}_l, l \in [k_c]\})$ for each class, fulfilling the assumptions above.
    Then for two nodes $v_i, v_j \in V$, with the same class $y_i = y_j$ and different discrete neighborhood distributions $\mathcal{D}_{p_i} \neq \mathcal{D}_{p_j}$ the expected distance between their MPNN embeddings $h_i$, $h_j$ is lower bounded by the distance of the means of the neighborhood distribution components:
    %\vspace{-0.5cm}
    \begin{align*}
        &\mathbb{E}\big[\norm{h_i - h_j}\big] \geq 
        \sigma_{min}(W) \norm{(\mathbb{E}_{
        x \sim (\mathcal{F}_c | c \sim D_{p_i})}[x]
        - \mathbb{E}_{
        x \sim (\mathcal{F}_c | c \sim D_{p_j})}[x]}
    \end{align*}
where $\sigma_{min}(W)$ denotes the smallest singular value of the learnable weight matrix $W$.

%Alternative, assuming you compute expectancy over all possible $c$'s, then it is exactly the conditional probability of F given $D_{pi}$ - or not?! ;):
%\begin{equation}
%        \mathbb{E}_{x \sim p(F \mid D_{pi})}\left[x\right] %= \mathbb{E}_{x \sim F_c} \left[ \mathbb{E}_{c \sim D_{pi}} \left[ x \right] \right]
%\end{equation}
\end{theorem}

\begin{proof}

By using Jensens inequality (see Lemma 1), we get: 
\begin{equation*}
    \mathbb{E}[\norm{h_i - h_j}] \geq \norm{\mathbb{E}[h_i - h_j]} = \norm{\mathbb{E}[h_i] - \mathbb{E}[h_j]} 
\end{equation*}
With the aggregation mechanism of the MPNN, we express the expectation of the embeddings in terms of the expectation of the neighbors and features:
\begin{align*}
    &= \norm{\mathbb{E}[\sum_{k \in N(i)} \frac{1}{deg(i)} W x_k] - \mathbb{E}[\sum_{l \in N(j)} \frac{1}{deg(j)} W x_l]} \\
    &= \norm{W (\frac{1}{deg(i)} \sum_{k \in N(i)} \mathbb{E}_{x_k \sim \mathcal{F}_c , c \sim D_{p_i}}[x_k] - \frac{1}{deg(j)} \sum_{l \in N(j)} \mathbb{E}_{x_l \sim \mathcal{F}_c , c \sim D_{p_j}}[x_l])} \\
    &=\norm{W (\mathbb{E}_{x \sim \mathcal{F}_c , c \sim D_{p_i}}[x] - \mathbb{E}_{x \sim \mathcal{F}_c , c \sim D_{p_j}}[x])} \\
    &\geq \sigma_{min}(W) \norm{\mathbb{E}_{x \sim \mathcal{F}_c , c \sim D_{p_i}}[x] - \mathbb{E}_{x \sim \mathcal{F}_c , c \sim D_{p_j}}[x]} \\
    %looking into expectations
    %& = \sigma_{min}(W) \norm{\sum_{c} \int_x x p_{x \sim \mathcal{F}_c , c \sim D_{p_j}}(x,c) - \sum_{c} \int_x x p_{\mathcal{F}_c , c \sim D_{p_j}}(x,c)}\\
    %rewrite as conditional prob
    %& = \sigma_{min}(W) \norm{ \sum_{c} \int_x xp_{c \sim D_{p_i}}(c|x)p_{x \sim \mathcal{F}_c}(x) - \sum_{c} \int_x x p_{c \sim D_{p_j}}(c|x)p_{x \sim \mathcal{F}_c}(x)}\\
   % &= \sigma_{min}(W) \norm{\int_x x p_{x \sim \mathcal{F}_c}(x) \big( \sum_c p_{c \sim D_{p_i}}(c|x) - p_{c \sim D_{p_j}} (c|x) \big)} \\
    %&\overset{?}{=} \sigma_{min}(W) \norm{\mathbb{E}_{x \sim \mathcal{F}_c}[x ]}TV(D_{p_i, D_{p_j}})
\end{align*}

Note that we used Lemma 2 for the last step and $\sigma_{min}$ is the smallest singular value of $W$.

\end{proof}

%\begin{itemize}
%    \item Formalize notion of different Neighborhood distributions in inside a class
%    \item TV <= KL, JS
%    \item Theorem as theoretical result
%    \item GMM Experiment to verify it
%    \item May train MoE MPNN (Nexperts = k*C) on small datast to show it performs better
%\end{itemize}

\section{Homophily Measures}
\label{app:hom}
We present the homophily and related measures used throughout this paper. 
To measure homophily and neighborhood consistency, we use three different measures, edge homophily, adjusted homophily and Label Informativeness (LI).

Edge homophily~\cite{H2GNN} counts the ratio of edged between nodes of the same class to all edges in the graph.
It is defined by 

\begin{equation}
    h_{edge} := \frac{|\{(u,v) : (u,v) \in \gE, y_u=y_v \}|}{|\gE|},
\end{equation}
where $y_u$ is the class of node $u$ and $y_v$ the class node $v$.

Adjusted homophily~\cite{LabelInformativeness} improves edge homophily my considering the number of classes, which makes the measure comparable across different datasets.
It is defined by

\begin{equation}
    h_{adj} := \frac{h_{edge} - \sum_{c=1}^C \overline{p}(c)^2}{1 - \sum_{c=1}^C \overline{p}(c)^2},
\end{equation}
where $h_{edge}$ is the edge homophily, and $\overline{p}(\cdot)$ is the degree weighted distribution of classes, \ie $\overline{p}(c) = \frac{\sum_{v:y_v = c} d(v)}{2|\gE|}$, where $d(v)$ is the degree of $v$.

Both, edge homophily and adjusted homophily measure only how nodes of the same classes are connected to each other, but ignore the similarity of neighborhood distributions of classes. 
Label informativness solves this issue by measuring the information a label of a node provides about the label of its neighbor

\begin{equation}
    LI := I(y_v, y_w)/H(y_v) = - \frac{
    \sum_{c_1, c_2} p(c_1, c_2) \log(\frac{p(c_1, c_2)}{ \overline{p}(c_1)  \overline{p}(c_2)})
    }{
    \sum_c  \overline{p}(c) \log( \overline{p}(c))
    },
\end{equation}

where $I(y_v, y_w)$ is the mutal information between the label of node $v$ and $w$ and $p(c_1, c_2) = \sum_{(v,w) \in E} \frac{\mathbbm{1}\{y_v = c_1, y_w=c_2\}}{|E|}$.

\section{Gradient Estimation of the Loss}
\label{app:details_gumbel}
We present the details on computing the gradient with respect to the parameters of the Message Passing Neural Network and the edge model given our loss function:

\begin{equation*}
    L(A(G), X, y;w,u) = \mathbb{E}_{\tilde{A} \sim p_{\mathbf{ \bm{\theta}}}}[l(f_w (\tilde{A}, X), y)], \text{ with } \bm{\theta} = g_u(X,E)
\end{equation*}

The gradient with respect to the parameter $w$ is given by chain rule with:

\begin{equation*}
    \nabla_w L(A(G), X, y;w,u) = \mathbb{E}_{\tilde{A} \sim p_{\mathbf{ \bm{\theta}}}} [\partial_w f_w^T (\tilde{A}, X) \nabla_{\hat{y}} l(\hat{y}, y)].
\end{equation*}

The gradient for the edge-model $g_u$ is: 

\begin{equation*}
    \nabla_u L(A(G), X, y;w,u) = \partial_u g_u(A(G), X)^T \nabla_{\bm{\theta}} L(A(G), X, y;w,u),
\end{equation*}
where $\nabla_{\bm{\theta}}L(A(G), X, y;w,u) = \nabla_{\bm{\theta}}\mathbb{E}_{\tilde{A} \sim p_{\mathbf{ \bm{\theta}}}}[l(f_w (\tilde{A}, X), y)], \text{ with } \bm{\theta} = g_u(X,E)$ is the loss gradient w.r.t. $\bm{\theta}$.
To compute the gradient one needs to draw Monte-Carlo samples from $p_{\bm{\theta}}$. 
However, the discrete sampling prevents the application of backpropagation to compute the gradient.
One possibility is to apply the score function estimator or REINFORCE estimator~\cite{reinforce}. 
In our setting, where $\bm{\theta}$ is Bernoulli distributed, we can use the Gumbel-Softmax estimator~\cite{gumbel}, which gives us a differentiable model by reparametrization that can be annealed into the final Bernoulli distribution.

\section{Real-World Graph Results with Standard Deviation}
\label{app:real_world_results_with_std}
The results for real-world graph datasets with standard deviations can be found in Table~\ref{tab:results_real-world_with_std}.

\renewcommand{\mytextsubscript}[1]{{\color{black}~\textsubscript{#1}}}

\begin{table*}[!ht]
    \centering
    %\caption{Measures of neighborhood distributions on the global graph and the relabeled classes after \textit{clustering on the entire dataset (100\% of the labels}, aka train/val/test splits). Thus, this setup simulates a ``perfect'' splitting of the classes. 
    %For training the models, however, we use the train/val splits of the dataset.}
    \caption{Node classification performance on the homophilic and heterophilic real-world datasets with standard deviations.} %check std error versus std dev
     \begin{adjustbox}{width=\textwidth}
     
    \begin{tabular}{l|cccccccccccc}
    
    \toprule 
                   & Cora & CiteSeer & PubMed & OGBn-ArXiv &Squirrel & Chameleon & Actor & Roman-Empire & Questions & Minesweeper & Tolokers & Amazon-ratings\\
    \midrule
     MLP  & $75.61\mytextsubscript{0.20}$ & $72.89\mytextsubscript{0.10}$ & $87.30\mytextsubscript{0.04}$ &$55.06\mytextsubscript{0.09}$ & $36.17\mytextsubscript{0.15}$ & $49.62\mytextsubscript{0.24}$ & $36.50\mytextsubscript{0.13}$ & $65.88\mytextsubscript{0.04}$ & $70.82\mytextsubscript{0.10}$ & $50.38\mytextsubscript{0.12}$ & $73.62\mytextsubscript{0.20}$ & $46.68\mytextsubscript{0.13}$\\
     %GCN-2 & ... \\
     GCN & $ 87.47\mytextsubscript{0.12}$ & $76.10\mytextsubscript{0.15}$ & $88.07\mytextsubscript{0.04}$ & $71.27\mytextsubscript{0.17}$ & $\textbf{55.80}\mytextsubscript{0.13}$  & $63.70\mytextsubscript{0.21}$ & $35.21\mytextsubscript{0.07}$ & $78.82\mytextsubscript{0.09}$ & $76.07\mytextsubscript{0.24}$ & $91.50\mytextsubscript{0.06}$ & $80.75\mytextsubscript{0.12}$ & $49.65\mytextsubscript{0.13}$ \\
     GraphSAGE & $87.89\mytextsubscript{0.13}$ & $76.12\mytextsubscript{0.14} $ & $88.22\mytextsubscript{0.04}$ & $69.65\mytextsubscript{0.27}$ & $42.46\mytextsubscript{0.14}$  & $62.99\mytextsubscript{0.18}$ & $36.79\mytextsubscript{0.10}$ & $81.78\mytextsubscript{0.06}$ & $75.38\mytextsubscript{0.09}$ & $\textbf{93.58}\mytextsubscript{0.05}$ & $80.16\mytextsubscript{0.12}$ & $50.61\mytextsubscript{0.10}$ \\
     GAT & $87.49\mytextsubscript{0.15}$ & $\textbf{76.43}\mytextsubscript{0.10}$ & $88.51\mytextsubscript{0.04}$ & $71.97\mytextsubscript{0.12}$ & $55.37\mytextsubscript{0.29}$  & $63.17\mytextsubscript{0.20}$  & $34.91\mytextsubscript{0.12}$ & $81.45\mytextsubscript{0.09}$ & $72.54\mytextsubscript{0.10}$ & $92.38\mytextsubscript{0.09}$ & $80.99\mytextsubscript{0.17}$ & $\textbf{50.56}\mytextsubscript{0.06}$\\
     GCN with DropEdge  & $87.93\mytextsubscript{0.13}$ & $75.99\mytextsubscript{0.15}$ & $88.64\mytextsubscript{0.03}$ & $71.75\mytextsubscript{0.16}$ & $53.59\mytextsubscript{0.17}$ &  $63.51\mytextsubscript{0.26}$ & $33.59\mytextsubscript{0.10}$ & $78.26\mytextsubscript{0.11}$ & $55.92\mytextsubscript{0.11}$ & $89.95\mytextsubscript{0.07}$ & $81.49\mytextsubscript{0.11}$  & $49.84\mytextsubscript{0.13}$ \\
     SDRF with GCN  & $87.88\mytextsubscript{0.13}$ & $76.12\mytextsubscript{0.14}$ & $88.05\mytextsubscript{0.04}$ & OOT  & $55.66\mytextsubscript{0.13}$  & $35.88\mytextsubscript{0.15}$ & $35.15\mytextsubscript{0.05}$ & $80.60\mytextsubscript{0.07}$ & $57.16\mytextsubscript{0.14}$ & $89.93\mytextsubscript{0.06}$ & $80.90\mytextsubscript{0.11}$ & $44.03\mytextsubscript{0.18}$
 \\
     \hline
     \textit{Gumbel-MPNN (ours)}  & $\textbf{87.95}\mytextsubscript{0.10}$  & $75.54\mytextsubscript{0.17}$ & $\textbf{88.73}\mytextsubscript{0.04}$ & $70.44\mytextsubscript{0.19}$ &$\textit{54.47}\mytextsubscript{0.12}$ & $\textbf{64.49}\mytextsubscript{0.25}$  & $\textbf{37.13}\mytextsubscript{0.1}$ & $\textbf{82.23}\mytextsubscript{0.07}$
 & $\textbf{76.61}\mytextsubscript{0.12}$  & $91.18\mytextsubscript{0.13}$  & $\textbf{81.56}\mytextsubscript{0.41}$  & $50.08\mytextsubscript{0.06}$ \\       
%Gumbel-GCN-2 & .... \\
     
     \bottomrule
    \end{tabular}
    \end{adjustbox}
    
    \label{tab:results_real-world_with_std}
\end{table*}

\section{Ablation Study}

We conduct an ablation study on the regularization terms of the loss function.
As discussed in Section~\ref{sec:methods}, we only selected one regularization term per dataset based on pre-experiments. 
For the ablation study, we reduce the loss function to the cross entropy loss.
The results are presented in Table~\ref{tab:ablation}.

\begin{table}[!ht]
    \centering
    \caption{Ablation study of the regularization term for Gumbel-MPNN.}
    \label{tab:ablation}
     \begin{adjustbox}{width=\textwidth}
    \begin{tabular}{c|cccccccccccc}
    \toprule
       & Cora & CiteSeer & PubMed & OGBn-ArXiv &Squirrel & Chameleon & Actor & Roman-Empire & Questions & Minesweeper & Tolokers & Amazon-ratings \\
    \midrule
     G.-MPNN w.o. Reg.   & $87.91_{0.13}$ & $75.55_{0.17}$ & $88.77_{0.04}$ & & $54.43_{0.1}$ & $57.26_{0.71}$ & $34.92_{0.11}$ & $77.28_{0.28}$ & $75.02_{0.27}$ & $90.15_{0.26}$ & $81.41_{0.39}$ & $49.93_{0.04}$\\ 
     Gumbel-MPNN    & $87.95\mytextsubscript{0.10}$  & $75.54\mytextsubscript{0.17}$ & $88.73\mytextsubscript{0.04}$ & $70.44\mytextsubscript{0.19}$ &$\textit{54.47}\mytextsubscript{0.12}$ & $64.49\mytextsubscript{0.25}$  & $37.13\mytextsubscript{0.1}$ & $82.23\mytextsubscript{0.07}$
 & $76.61\mytextsubscript{0.12}$  & $91.18\mytextsubscript{0.13}$  & $81.56\mytextsubscript{0.41}$  & $50.08\mytextsubscript{0.06}$  \\
    \bottomrule
    \end{tabular}
    \end{adjustbox}
\end{table}

\section{Squirrel and Chameleon Filtered}
\label{app:squirrel_chameleon}
Since the work of \citet{HeteroPhilDSNew} the use of the original Squirrel and Chameleon dataset is highly disputed, since it has been recognized that there nodes with identical features and class in training and test set. 
\citet{HeteroPhilDSNew} argue that these duplicates are an error in the preprocessing. 
However, it is also possible that connected nodes in a wikipedia network do share the same subset of nouns and, therefore, the same set of node features.
Nevertheless, it could be a train-test leakage, which is not too problematic, as long as these dataset are not the only datasets used for evaluation, it measures a models capability to memorize a subset of the training set, which can also be an important property.
This being said, we also compare our models on the filtered version provided by \citet{HeteroPhilDSNew}, where all duplicates have been removed. 
The result is presented in Table~\ref{tab:filtered_ds}.

\begin{table}[!ht]
    \centering
    \caption{Node classification on the Squirrel and Chameleon dataset with filtered duplicates.}

    \begin{adjustbox}{width=\textwidth}

    \begin{tabular}{l|ccccccccc}
     \toprule
      & MLP & GCN & GraphSage & GAT & GCN with DropEdge & SDRF with GCN & Gumbel MPNN\\
    \midrule
     Chameleon-unfiltered    & $49.62\mytextsubscript{0.24}$ & $63.70\mytextsubscript{0.21}$ & $62.99\mytextsubscript{0.18}$ & $63.17\mytextsubscript{0.20}$ & $63.51\mytextsubscript{0.26}$ & $63.82\mytextsubscript{0.06}$ & $\textbf{64.49}\mytextsubscript{0.25}$ \\
     
      Squirrel-unfiltered    & $36.17\mytextsubscript{0.15}$ & $55.80\mytextsubscript{0.13}$ & $42.46\mytextsubscript{0.14}$ & $55.37\mytextsubscript{0.29}$ & $53.59\mytextsubscript{0.17}$ & $55.66\mytextsubscript{0.13}$ & $\textit{54.47}\mytextsubscript{0.12}$ \\
     \bottomrule
    \end{tabular}
    \end{adjustbox}

    \label{tab:filtered_ds}
\end{table}

In contrast to the unfiltered Squirrel and Chameleon dataset, we observe that MLP is the strongest Method on Squirrel.
This means that the benefit of the duplicates is mostly a result of the neighbor aggregation, beside that the neighbors do not add additional information. 
For Chameleon, the ranking is similar to the unfiltered dataset, although all models decreased about $20\%$ in accuracy.

\section{Neighborhood Distribution Clustering by Gaussian Mixture Models}
\label{app:GMM_Clust}

In this section, we describe the clustering of the neighborhood distributions.
We calculate the empirical 1-hop distribution $p_i=\hat{p}_{N(v_i)}$ of each node $v_i$.
All $p_i$ are clustered separated by classes with a Gaussian Mixture Model (GMM).
We fit and evaluate number of components from $k \in \{1, \dots, 25\}$, \ie each class is separated into $1$ to $25$ classes based on the components found by the GMM in the neighborhood distribution. 
For each class, we select the $k_c^*$ that maximized the Bayesian Information Criteria (BIC) for the model. 
The overall new number of classes per dataset is given in Table~\ref{tab:distribution_comparison}.
We trained MPNNs on the new set of broken up classes and mapped the result back to original class, which did non improve the model performance, as explained in Section~\ref{sec:NeighDistImpact}.

The Procedure is presented in the following algorithm. 

\begin{algorithm}
\caption{Gaussian Mixture Model Class Decomposition}
\begin{algorithmic}
\REQUIRE Graph $\gG=(\gV,\gE)$, set of classes $\gC$, maximum number of components per class $k_c$

\FOR{$c = 0$ to $|\gC|-1$}
    \FOR{$l = 0$ to $k_c$}
      \STATE $P_i^c \leftarrow$ Set of empirical 1-hop distribution for each node for class $c$
      \STATE $BIC_c^l, \tilde{Y}_c^l \leftarrow$ GMM($P_i^c$, $l$)
    \ENDFOR
    \STATE $l^* \leftarrow \argmax_l(BIC_c^l)$
\ENDFOR

\STATE \textbf{Output:} $[Y_1^{l^*}, \dots, Y_{|\gC| - 1}^{l^*}]$
\end{algorithmic}
\end{algorithm}

\section{Datasets}
\label{app:dataset}
For our real-world dataset, we use the homophilic Planetoid citation graphs Cora, CiteSeer~\cite{CoraCiteSeer}, and PubMed~\cite{PubMed}, and ogbn-arxiv from the ogb benchmark~\cite{ogb-benchmarks}.
We use the random Planetoid split suggested by \citet{PitfallsGNNEval}, where we randomly sample 20 nodes per class for training, 30 nodes per class for validation and all other for testing. 
For ogbn-arxiv, we use the default train, validation test split, where we train on all nodes up to year $2017$, validate on all nodes from year $2018$, and test on all nodes from year $2019$.
As the heterophilic datasets, we use the Wikipedia graphs Squirrel, Chameleon~\citet{SquirrelChameleon}, and Actor~\cite{GeomGNN} with the $10$ default splits per dataset~\cite{GeomGNN}.
Additionally, we include the recent heteophilic benchmark datasets proposed by \citet{HeteroPhilDSNew} consisting of Minesweeper, Roman-Empire, Amazon-ratings, Tolokers, and Questions. 
For descriptive numbers of the dataset, see Table~\ref{tab:ds}.

\begin{table}[!ht]
    \centering
    \caption{The number of nodes $|V|$, the number of edges $E$, the feature dimension $d$, the number of classes $|\mathcal{C}|$, and the edge homophily, adjusted homophily, and Label Informativness (LI).}. 
     \begin{adjustbox}{width=0.6\textwidth}
    \begin{tabular}{l|rrrr|rrr}
    \toprule
    Dataset & $|V|$ & $|E|$ & $d$ & $|\mathcal{C}|$ & $h_{edge}$ & $h_{adj}$ & $LI_\text{edge}$ \\
    \midrule
    Cora & $2,708$ & $5,278$ & $1,433$ & $7$ & $0.81$ & $0.77$ & $0.59$ \\
    CiteSeer & $3,327$ & $4,552$ & $3,703$ & $6$ &$0.74$ & $0.67$ & $0.45$ \\
    PubMed & $19,717$ & $44,324$ & $500$ & $3$ &$0.80$ & $0.69$ & $0.41$ \\
    ogbn-arxiv & $169,343$ & $1,166,243$ & $128$ & $40$ &$0.66$ & $0.59$ & $0.46$\\
    \midrule
    Squirrel & $5201$ & $217073$ & $2089$ & $5$ & $0.22$ & $0.01$ & $0.00$ \\
    Chameleon & $2277$ & $36101$ & $2325$ & $5$ &$0.24$ & $0.04$ & $0.04$  \\
    Actor & $7,600$ & $30,019$ & $932$ & $5$ & $0.22$ & $0.01$ & $0.00$ \\
    Roman-Empire & $22,662$ & $32,927$ & $300$ & $18$ &$0.05$ & $-0.05$ & $0.11$  \\
    Questions & $48,921$ & $153,540$ & $301$ & $2$ &$0.84$ & $0.02$ & $0.00$  \\
    Minesweeper & $10,000$ & $39,402$ & $7$ & $2$ &$0.68$ & $0.01$ & $0.00$  \\
    Tolokers & $11,758$ & $519,000$ & $10$ & $2$ &$0.59$ & $0.09$ & $0.01$  \\
    Amazon-ratings & $24,492$ & $93,050$ & $300$ & $5$ & $0.38$ & $0.14$ & $0.03$\\
    \bottomrule
    \end{tabular}
    \end{adjustbox}
    
    \label{tab:ds}
\end{table}

\section{Hyperparameter}
\label{app:hyperparam}

We tune the hyperparameters for each model and each dataset on the regular node classification task.
For the other experiments, we reuse these parameters.
All models are optimized with Adam~\cite{adam} and a learning rate of $0.001$.
We tune for all models the hidden dimension, the use of residual connections, and whether the graph is undirected or directed. 
The final hyperparameters are presented in Table~\ref{tab:hp_mlp}, Table~\ref{tab:hp_gcn}, Table~\ref{tab:hp_gat} and Table~\ref{tab:hp_sage}.

\begin{table}[!ht]
    \centering
    \caption{Final hyperparameters of MLP for each dataset.}

    \begin{adjustbox}{width=\textwidth}

    \begin{tabular}{l|rrrrrrrrrrrr}
    \toprule
         &  Cora & CiteSeer & PubMed & OGBn-ArXiv & Squirrel & Chameleon & Actor & Roman-Empire & Questions & Minesweeper & Tolokers & Amazon-ratings\\
     \midrule
     Hidden dim  & $1024$ & $2048$ & $1024$ & $1024$ & $256$ & $1024$ & $2048$ & $128$ & $1024$ & $1024$ & $128$ & $256$\\
     Layers  & $1$ & $1$ & $2$ & $1$ & $3$ & $5$ & $4$ & $5$ & $3$ & $1$  & $3$ & $2$\\
     Dropout  & $0.8$ & $0.9$ & $0.8$ & $0.1$ & $0.3$ & $0.0$ & $0.0$ & $0.1$ & $0.8$ & $0.4$ & $0.2$ & $0.1$\\
     Layernorm & yes & yes & no & yes & no & yes & yes & yes & yes & no  & yes & yes\\
     Residual  & yes & yes & yes & yes & yes & yes & yes & yes & no & yes  & no & yes\\
     \bottomrule
    \end{tabular}
    \end{adjustbox}
    \label{tab:hp_mlp}
\end{table}

\begin{table}[!ht]
    \centering
    \begin{adjustbox}{width=\textwidth}
    \begin{tabular}{l|rrrrrrrrrrrr}
    \toprule
         &  Cora & CiteSeer & PubMed & OGBn-ArXiv & Squirrel & Chameleon & Actor & Roman-Empire & Questions & Minesweeper & Tolokers & Amazon-ratings\\
     \midrule
     Hidden dim  & $256$ & $2048$ & $256$ & $1024$ & $128$ & $2048$ & $128$ & $128$ & $512$ & $256$ & $128$ & $1024$\\
     Layers  & $2$ & $2$ & $3$ & $5$ & $5$ & $2$ & $1$ & $2$ & $5$ & $4$  & $3$ & $5$\\
     Dropout  & $0.1$ & $0.9$ & $0.3$ & $0.2$ & $0.2$ & $0.1$ & $0.8$ & $0.2$ & $0.8$ & $0.1$ & $0.4$ & $0.2$\\
     Layernorm & yes & yes & yes & yes & yes & no & no & no & no & yes  & yes & yes\\
     Residual  & no  & no & yes & yes & yes & no & yes & yes & yes & yes  & yes & yes\\
    Undirected  & yes & yes & yes & yes & yes & yes & yes & no & yes & yes  & no & no\\
     \bottomrule
    \end{tabular}
    \end{adjustbox}
    \caption{Final hyperparameters of GCN for each dataset.}
    \label{tab:hp_gcn}
\end{table}

\begin{table}[!ht]
    \centering
    \caption{Final hyperparameters of GraphSAGE for each dataset.}

    \begin{adjustbox}{width=\textwidth}
    \begin{tabular}{l|rrrrrrrrrrrr}
    \toprule
         &  Cora & CiteSeer & PubMed & OGBn-ArXiv & Squirrel & Chameleon & Actor & Roman-Empire & Questions & Minesweeper & Tolokers & Amazon-ratings\\
     \midrule
     Hidden dim  & $512$ & $256$ & $1024$ & $1024$ & $1024$ & $2048$ & $256$ & $1024$ & $512$ & $2048$ & $2048$ & $512$\\
     Layers  & $2$ & $2$ & $2$ & $2$ & $5$ & $4$ & $2$ & $2$ & $5$ & $5$  & $2$ & $3$\\
     Dropout  & $0.8$ & $0.2$ & $0.2$ & $0.2$ & $0.7$ & $0.8$ & $0.2$ & $0.2$ & $0.8$ & $0.8$ & $0.7$ & $0.7$\\
     Layernorm & yes & yes & no & yes & yes & yes & no & no & yes & yes  & yes & yes\\
     Residual  & no  & no & no & yes & yes & yes & no & np & yes & yes  & yes & yes\\
    Undirected  & yes & yes & yes & yes & yes & yes & no & no & yes & yes  & no & yes\\
     \bottomrule
    \end{tabular}
    \end{adjustbox}
    \label{tab:hp_sage}
\end{table}

\begin{table}[!ht]
    \centering
    \caption{Final hyperparameters of GAT for each dataset.}

    \begin{adjustbox}{width=\textwidth}
    \begin{tabular}{l|rrrrrrrrrrrr}
    \toprule
         &  Cora & CiteSeer & PubMed & OGBn-ArXiv & Squirrel & Chameleon & Actor & Roman-Empire & Questions & Minesweeper & Tolokers & Amazon-ratings\\
     \midrule
     Hidden dim  & $512$ & $512$ & $128$ & $256$ & $256$ & $64$ & $256$ & $128$ & $128$ & $32$ & $128$ & $256$\\
     Num. Heads & $1$ & $1$ & $16$ & $4$ & 1 & 8 & 1 & 8 & 1 & 2 & 8 & 1\\
     Layers  & $3$ & $1$ & $3$ & $5$ & $3$ & $4$ & $3$ & $3$ & $3$ & $5$  & $4$ & $3$\\
     Dropout  & $0.5$ & $0.8$ & $0.7$ & $0.3$ & $0.4$ & $0.2$ & $0.8$ & $0.4$ & $0.5$ & $0.1$ & $0.2$ & $0.2$\\
     Layernorm & yes & no & yes & yes & no & no & no & yes & no & no  & yes & yes\\
     Residual  & no  & yes & yes & yes & no & yes & yes & yes & yes & yes  & yes & yes\\
    Undirected  & yes & yes & yes & yes & yes & yes & yes & yes & no & yes  & yes & yes\\
     \bottomrule
    \end{tabular}
    \end{adjustbox}
    \label{tab:hp_gat}
\end{table}

For the Gumbel-Softmax approach, we conducted pre-experiments to choose the candidate selection strategy and the regularization term. 
The result can be seen in Table~\ref{tab:hp_cand}.

\begin{table}[!ht]
    \centering
    \caption{Candidate selection strategy and regularization term for the Gumbel-Softmax approach for each dataset.}
    \begin{adjustbox}{width=\textwidth}
    \begin{tabular}{l|rrrrrrrrrrrr}
    \toprule
         &  Cora & CiteSeer & PubMed & OGBn-ArXiv & Squirrel & Chameleon & Actor & Roman-Empire & Questions & Minesweeper & Tolokers & Amazon-ratings\\
     \midrule
     Candidate sel.  & global & local & 2-hop & global & 2-hop & global & local & 2-hop & global & ricci & ricci & 2-hop\\
     Reg. term       & lc & mix & mix & pc & pc & - & mix & dc & pc & pc  & none & mix\\
     \bottomrule
    \end{tabular}
    \end{adjustbox}

    \label{tab:hp_cand}
\end{table}

\begin{table}[!ht]
    \centering
    \caption{Final hyperparameters of Gumbel-MPNN for each dataset.}

    \begin{adjustbox}{width=\textwidth}
    \begin{tabular}{l|rrrrrrrrrrrr}
    \toprule
         &  Cora & CiteSeer & PubMed & OGBn-ArXiv & Squirrel & Chameleon & Actor & Roman-Empire & Questions & Minesweeper & Tolokers & Amazon-ratings\\
     \midrule
     Hidden dim  & $1024$ & $256$ & $128$ & $64$ & $32$ & $256$ & $128$ & $64$ & $128$ & $32$ & $128$ & $320$\\
     Layers  & $2$ & $1$ & $3$ & $5$ & $5$ & $5$ & $2$ & $2$ & $2$ & $5$  & $3$ & $2$\\
     Dropout  & $0.8$ & $0.7$ & $0.9$ & $0.0$ & $0.3$ & $0.0$ & $0.4$ & $0.2$ & $0.0$ & $0.0$ & $0.0$ & $0.4$\\
     Reg. weight. & $0.05$ & $0.001$ & $0.001$ & $0.1$ & $0.05$ & $0.5$ & $0.01$ & $0.01$ & $0.05$ &$ 0.01$ & $0.0$\\
     Layernorm & yes & yes & yes & yes & yes & no & no & no & no & yes  & yes & yes\\
     Residual  & no  & yes & yes & yes & yes & no & yes & yes & yes & yes  & yes & yes\\
    Undirected  & yes & yes & yes & yes & yes & yes & yes & no & yes & yes  & yes & no\\
     \bottomrule
    \end{tabular}
    \end{adjustbox}
    \label{tab:hp_gumbel}
\end{table}

%\section{Inference Times}

\section{Notation}
\label{app:notation}

The notation and varialbles used in this paper are summarized in Table~\ref{tab:notation}.

\begin{table}[!ht]
    \centering
    \begin{tabular}{ll}
    \toprule
       Variable  & Definition \\
    \midrule
     $\gG$    & Graph consisting of nodes and edges \\
     $\gV$    & All nodes of the graph $G$ \\
     $\gV_c$ & All nodes that belong to class $c$ \\
     $\gV_{train}$  & Labeled subset of the nodes used for training\\
     $\gV_{trans}$ &  Unlabeled subset of the nodes\\
     $\gE$    & All edges of the graph $G$ \\
     $\gA(\gG)$ & The adjacency matrix of graph $\gG$ \\
     $\gC$ & Set of all classes \\
     $c$ & A specific class $c \in \gC$ \\
     $\mX$ & Feature matrix containing the feature vectors for all nodes \\
     $\gF_c$ & The feature distribution of class $c$\\
     $D_c$ & The neighborhood-mixture distribution of class $c$ \\
     $D_l$ & A component of a neighborhood distribution \\
     $\hat{p}_{N(v_i)}$ & Empirical 1-hop distribution of node $i$ \\
     $\pi_i$ & The weight for the categorical distribution $i$ in the mixture model \\
     $k_c$ & The number of components of the neighborhood distribution of class $c$ \\
     $f$ & The downstream model to classify the nodes \\
     $w$ & Trainable weights of the model $f$ \\
     $g$ & Edge model to predict the parameters of the edge distribution \\
     $u$ & Trainable weights of the model $g$ \\
     $v_i$ & A single nodes of $v_i \in \gV $ \\
     $x_i$ & The feature vector associated with node $v_i$\\
     $y_i$ & The label associated with node $v_i$\\
     $x$ & Feature vector associated with a node $v$ \\

    \bottomrule
    \end{tabular}
    \caption{Variables used throughout the paper along with their meaning.}
    \label{tab:notation}
\end{table}

\end{document}